\documentclass{article}

\newcommand{\RNum}[1]{\uppercase\expandafter{\romannumeral #1\relax}}

\usepackage{makecell}
\usepackage{multirow}

\usepackage{soul}
\usepackage{xcolor}

\usepackage{hyperref}

\usepackage{microtype}
\usepackage{graphicx}
\usepackage{subfigure}
\usepackage{booktabs} 

\usepackage{hyperref}

 \usepackage[accepted]{icml2024}

\usepackage{amsmath}
\usepackage{amssymb}
\usepackage{mathtools}
\usepackage{amsthm}

\usepackage[capitalize,noabbrev]{cleveref}

\theoremstyle{plain}
\newtheorem{theorem}{Theorem}[section]

\newtheorem{lemma}[theorem]{Lemma}

\theoremstyle{definition}
\newtheorem{definition}[theorem]{Definition}

\theoremstyle{remark}

\usepackage[textsize=tiny]{todonotes}

\icmltitlerunning{Projection-Free Online Convex Optimization with Time-Varying Constraints}

\begin{document}

\newcommand{\ben}[1]{\textcolor{violet}{\{BK: #1\}}}
\newcommand{\brac}[1]{\left(#1\right)}
\newcommand{\enorm}[1]{\left\Vert#1\right\Vert}
\newcommand{\matnorm}[2]{\left\Vert#1\right\Vert_{#2}}
\newcommand{\ceil}[1]{\left\lceil#1\right\rceil}

\def\vz{{\textbf{0}}}
\def\g{{\mathbf{g}}}
\def\m{{\mathbf{m}}}
\def\r{{\mathbf{r}}}
\def\x{{\mathbf{x}}}
\def\e{{\mathbf{e}}}
\def\a{{\mathbf{a}}}
\def\u{{\mathbf{u}}}
\def\v{{\mathbf{v}}}
\def\z{{\mathbf{z}}}
\def\w{{\mathbf{w}}}
\def\y{{\mathbf{y}}}
\def\p{{\mathbf{p}}}
\def\b{{\mathbf{b}}}
\def\X{{\mathbf{X}}}
\def\Y{{\mathbf{Y}}}
\def\A{{\mathbf{A}}}
\def\M{{\mathbf{M}}}
\def\I{{\mathbf{I}}}
\def\B{{\mathbf{B}}}
\def\C{{\mathbf{C}}}
\def\V{{\mathbf{V}}}
\def\Z{{\mathbf{Z}}}
\def\W{{\mathbf{W}}}
\def\U{{\mathbf{U}}}
\def\Q{{\mathbf{Q}}}
\def\P{{\mathbf{P}}}
\def\H{{\mathbf{H}}}
\def\S{{\mathbf{S}}}
\def\bSigma{{\mathbf{\Sigma}}}
\def\EV{{\mathbf{EV}}}

\def\matE{{\mathbf{E}}}

\newcommand{\mM}{\mathcal{M}}
\newcommand{\mT}{\mathcal{T}}
\newcommand{\mW}{\mathcal{W}}
\newcommand{\mX}{\mathcal{X}}
\newcommand{\mA}{\mathcal{A}}
\newcommand{\R}{\mathcal{R}}
\newcommand{\mP}{\mathcal{P}}
\newcommand{\mF}{\mathcal{F}}
\newcommand{\mE}{\mathcal{E}}
\newcommand{\mC}{\mathcal{C}}
\newcommand{\mK}{\mathcal{K}}
\newcommand{\mS}{\mathcal{S}}
\newcommand{\ball}{\mathcal{B}}
\newcommand{\mbS}{\mathbb{S}}
\newcommand{\mD}{\mathcal{D}}
\newcommand{\E}{\mathbb{E}}
\newcommand{\Id}{\textbf{I}}

\newcommand{\pinv}{\dagger}

\newcommand{\gap}{\textrm{gap}}
\newcommand{\dist}{\textrm{dist}}
\newcommand{\diag}{\textbf{diag}}
\newcommand{\nnz}{\textrm{nnz}}
\newcommand{\trace}{\textrm{Tr}}
\newcommand{\rank}{\textrm{rank}}
\newcommand{\reals}{\mathbb{R}}

\newcommand{\supp}{\mathop{\mbox{\rm supp}}}

\newcommand{\oraclep}{\mathcal{O}_{\mP}}
\newcommand{\oraclek}{\mathcal{O}_{\mK}}

\twocolumn[
\icmltitle{Projection-Free Online Convex Optimization with Time-Varying Constraints}




\begin{icmlauthorlist}
\icmlauthor{Dan Garber}{yyy}
\icmlauthor{Ben Kretzu}{yyy}
\end{icmlauthorlist}

\icmlaffiliation{yyy}{Faculty of Data and Decision Sciences, Technion - Israel Institute of Technology, Haifa, Israel}

\icmlcorrespondingauthor{Ben Kretzu}{benkretzu@campus.technion.ac.il}
\icmlcorrespondingauthor{Dan Garber}{dangar@technion.ac.il}

\icmlkeywords{Machine Learning, ICML}

\vskip 0.3in
]



\printAffiliationsAndNotice{}  

\begin{abstract}
We consider the setting of online convex optimization with adversarial time-varying constraints in which actions must be feasible w.r.t. a fixed constraint set, and are also required on average to approximately satisfy additional time-varying constraints. Motivated by scenarios in which the fixed feasible set (hard constraint) is difficult to project on, we consider projection-free algorithms that access this set only through a linear optimization oracle (LOO). We present an algorithm that, on a sequence of length $T$ and using overall $T$ calls to the LOO, guarantees $\tilde{O}(T^{3/4})$ regret w.r.t. the losses and $O(T^{7/8})$ constraints violation (ignoring all  quantities except for $T$) . In particular, these bounds hold w.r.t. any interval of the sequence. We also present a more efficient algorithm that requires only first-order oracle access to the soft constraints and achieves similar bounds w.r.t. the entire sequence. We extend the latter to the setting of bandit feedback and obtain similar bounds (as a function of $T$) in expectation.
\end{abstract}

\section{Introduction} \label{sec:intro}
We consider a particular setting of the well studied paradigm for sequential prediction known as online convex optimization (OCO)  \cite{HazanBook, Shalev12}. In (standard) OCO, a decision maker is required to iteratively choose an action --- some point $\x_t$ on each round $t$ (the total number of rounds $T$ is finite and assumed  for simplicity to be known in advanced), which must belong to some fixed (throughout all rounds) feasible convex set $\mK\subset\reals^n$ \footnote{for ease of presentation we consider the underlying space to be $\reals^n$, however any finite-dimensional Euclidean space is suitable}, which we will also assume to be compact (as is often standard). After choosing $\x_t$, a scalar loss  $f_t(\x_t)$ is incurred, where $f_t:\reals^n\rightarrow\reals$ is some arbitrary convex loss function (assumed for sake of analysis to be chosen adversarially). The standard goal is, on the course of the $T$ rounds, to choose feasible actions $\x_1,\dots,\x_T$ such that the regret,  given by the difference $\sum_{t=1}^Tf_t(\x_t) - \min_{\x\in\mK}\sum_{t=1}^Tf_t(\x)$, grows (as a function of $T$) only at a sublinear rate (the slower the better). 

In the particular setting considered in this work --- \textit{OCO with time-varying constraints}, we assume that besides the \textit{hard and fixed constraint} given by  the set $\mK$ to which all played points must belong, there are additional \textit{soft and time-varying constraints} given by convex functions $g_1,\dots,g_T\in\reals^n\rightarrow\reals$, where  $g_t$ is revealed at the end of round $t$ and encodes the constraint $g_t(\x) \leq 0$. In this setting the standard goal is two folded: I. to guarantee sublinear regret  w.r.t. the loss functions $f_1,\dots,f_T$ against the best action in hindsight in the intersection of the hard constraint and all soft constraints, i.e., to guarantee that 
\begin{align}\label{eq:regret}
\sum_{t=1}^Tf_t(\x_t) - \min_{\x\in\mK:~\g_t(\x)\leq 0 ~\forall t\in[T]}\sum_{t=1}^Tf_t(\x) = o(T),
\end{align}
 and II. to guarantee that the cumulative violation of the soft constraints is also sublinear in $T$, i.e., that 
\begin{align}\label{eq:constraint}
 \sum_{t=1}^T\left\{g_t^+(\x_t):=\max\{g_t(\x_t),0\}\right\} = o(T).  
\end{align} 
Here, $g_t^+(\x) :=\max\{g_t(\x),0\}$  is introduced to prevent a natural undesired phenomenon in which strongly satisfying some constraints can compensate strongly violating others.    

This setting was recently studied in \cite{yi2022regret, NEURIPS2022_ec360cb7, NEURIPS2022_d9b56471, neely2017online, cao2018online}. The state-of-the-art  bounds with full-information (i.e., $f_t(\cdot),g_t(\cdot)$ become known to the algorithm at the end of each round $t$) for this setting are $O(\sqrt{T})$ for the regret (as given in \eqref{eq:regret}) and $O(T^{3/4})$ for the constraint violation (as given in \eqref{eq:constraint}), see for instance \cite{NEURIPS2022_ec360cb7}.

However, previous works require as a sub-routine to solve on each iteration a convex optimization problem that is at least as difficult as computing a Euclidean projection onto the feasible set $\mK$. In high-dimensional settings and when the set $\mK$ admits a non-trivial structure, this highly limits the applicability of the proposed methods. Motivated by this observation, in this work, and to the best of our knowledge for the first time, we consider so-called \textit{projection-free algorithms} for OCO with time-varying constraints. Concretely, motivated by vast work on projection-free methods for OCO in recent years, e.g., \cite{Hazan12, Garber22a, hazan2020faster, Garber23a, mhammedi2021efficient, mhammedi2022}, we consider algorithms that only access the feasible set $\mK$ through a linear optimization oracle (that is an oracle that given a linear function, can find a minimizer of it over $\mK$) and consider online algorithms that throughout the $T$ rounds make at most $T$ calls to this oracle.


\textbf{Example I: online minimum-cost capacitated flow.} 
Consider a fixed directed acyclic graph $G$ with $n$ nodes,  $m$ edges, source node $s$ and target node $e$. A decision maker (DM) must route on each round $t$ a unit  flow from $s$ to $e$, i.e., choose some point $\x_t$ in the corresponding \textit{unit flow polytope}\footnote{this is also the convex-hull of all identifying vectors of paths from node $s$ to node $e$ in the graph} $\mK$. The DM then incurs cost $f_t(\x_t)$. If each edge is associated with a linear cost, we may have $f_t(\x) =\mathbf{f}_t^{\top}\x$ for some $\mathbf{f}_t\in\reals^{m}$. The DM also needs to respect time-varying edge capacity constraints given by $\x_t(i) \leq \mathbf{c}_t(i)$, $i=1,\dots,m$, where $\mathbf{c}_t$ has non-negative entries. Here, the constraint function on time $t$ is simply $g_t(\x) = \max_{i\in[m]}\{\x(i) - \mathbf{c}_t(i)\}$.  The unit flow-polytope 
is difficult to project on, however linear optimization over it corresponds to a simple weighted shortest path computation, which takes linear time using dynamic programming.

\textbf{Example II: online semidefinite optimization.} Consider a sequence of $T$ semidefinite optimization problems over the bounded positive semidefinite (PSD) cone $\mK= \{\X\in\mbS^n~|~\X\succeq 0,~\trace(\X)\leq \tau\}$, where $\mbS^n$ denotes the space of $n\times n$ real symmetric matrices and $\tau >0$. Each instance $t$ in the sequence consist of a convex objective $f_t(\X):\mbS^n\rightarrow\reals$ and a set of $m_t$ linear inequalities $\trace(\A_{t,i}^{\top}\X)\leq b_{t,i}, i\in\{1,\dots,m_t\}$. Here the fixed hard constraints are given by the convex bounded PSD cone and the time-varying soft constraints can be encoded using $g_t(\X) = \max_{i\in[m_t]}\{\trace(\A_{t,i}^{\top}\X)-b_{t,i}\}$. Projecting onto the bounded PSD cone requires a full eigen-decomposition of a $n\times n$ matrix, which in practical implementations requires $O(n^3)$ runtime. Linear optimization over this set amounts to  a single leading eigenvector computation whose runtime scales only with $n^2$, and even faster when the input matrix is sparse, using fast iterative eigenvector methods. 

We believe this setting is well suited to capture important online scenarios with changing constraints: at one extreme, some works as the ones mentioned above, enforce the hard constraint $\mK$ using costly computational procedures (projections) which can be prohibitive, as in the examples above. At the other extreme, for the sake of computational efficiency one can also model membership to $\mK$ as a long term soft constraint (this was the motivation of \cite{mahdavi2012trading} which pioneered the study of OCO with long term constraints), however this might result in a too crude approximation of real-world scenarios. For instance, in Example 1 above it makes sense that even if the capacity constraints could be violated, the decision maker still must choose a feasible flow on each round, or that in Example II the chosen matrix must indeed be PSD. We view our work as an appealing  middle ground between these two extremes: distinguishing between hard and soft (varying) constraints, while insisting on tractable procedures for high-dimensional and complex domains.    

\subsection{Contributions (informally stated)} 
\textbf{I.} We give an algorithm that using no more than $T$ calls to the LOO of $\mK$ throughout all $T$ rounds, guarantees that on each interval $[s,e]$, $1\leq s < e \leq T$, the regret w.r.t. the set $\widehat{\mK}_{s,e} = \{\x\in\mK~|~g_t(\x) \leq 0 ~\forall s\leq t\leq e\}$ is upper-bounded by $\tilde{O}(T^{3/4})$  and the overall non-negative constraint violation on the interval $\sum_{t=s}^eg_t^+(\x_t)$, is upper-bounded by $O(T^{7/8})$ (ignoring all quantities except for $T$ and $n$). Note this metric (bounds w.r.t. any interval)\footnote{such metric is also known as \textit{adpative regret} and was proposed in \cite{hazan2009efficient}} is substantially stronger than the ones in \eqref{eq:regret}, \eqref{eq:constraint}, and in particular is  interesting even when there is no point  $\x\in\mK$ that satisfies all constraints $g_t(\x) \leq 0, t\in[T]$. See Section \ref{sec:d+p}.

\textbf{II.} We give a more efficient algorithm than the latter which only requires first-order access to $g_1,\dots,g_T$, and  provides similar bounds, but only w.r.t. to the entire sequence (i.e., standard  bounds in the sense of  \eqref{eq:regret}, \eqref{eq:constraint}). 
See Section \ref{sec:lagrange}.

\textbf{III.} Finally, we extend the latter algorithm to the  bandit setting in which only the scalar values $f_t(\x_t), g_t(\x_t)$ are observed on each iteration $t$, where $\x_t$ is the played point. The algorithm guarantees $\tilde{O}(\sqrt{n}T^{3/4})$ expected regret 
and expected non-negative constraint violation $\sum_{t=1}^Tg_t^+(\x_t)=\tilde{O}(n^{1/4}T^{7/8})$. See Section \ref{sec:lagrange:bandit}.

\section{Preliminaries}
\subsection{Assumptions and notation}
As stated, we assume that the feasible set $\mK\subset\reals^n$, which encodes the hard constraints (must be satisfied on each round $t$), is convex and compact and we let $R$ be such that $\mK\subseteq{}R\ball$, where $\ball$ denotes the origin-centered unit Euclidean ball. We also assume w.l.o.g. that $\mathbf{0}\in\mK$. We assume all loss and constraint functions $f_1,g_1,\dots,f_T,g_T$ are convex over $R\ball$, and we let $G_f$ and $G_g$ denote upper-bounds on the $\ell_2$ norm of the subgradients of each $f_t, t\in[T]$ and $g_t, t\in[T]$ over the ball $R\ball$, respectively. 

Our online algorithms will consider the $T$ prediction rounds in disjoint blocks (batches) of size $K$, where $K$ is an integer parameter. We assume throughout for simplicity and w.l.o.g. that $T/K$ is an integer. We also use the notation $m(t) :=  \lceil{t/K}\rceil$ to map from some $t\in[T]$ to the corresponding block index $m(t)\in[T/K]$. 

\subsection{Fast approximately-feasible projections via linear optimization}
In order to construct our projection-free algorithms that  rely only on solving linear optimization problems over the fixed convex set $\mK$ (as opposed to projection or even more complex  steps) we rely on the technique recently introduced in \cite{Garber22a, Garber23a} of fast computation of  \textit{approximately-feasible projections} using a LOO.
\begin{definition}[Approximately-feasible Projection Oracle\footnote{\cite{Garber23a} considers arbitrary matrix-induced norms. Here we only consider the Euclidean case which is equivalent to setting $\A=\I$ in their definition.}] \label{def:app_feasible_projection}
Given a convex and compact set $\mK\subset\reals^n$, and a tolerance $\epsilon > 0$, we say a function $\mathcal{O}_{AFP}(\y,\epsilon,\mK)$ is an \textit{approximately-feasible projection (AFP) oracle} (for the set $\mK$ with parameter $\epsilon$), if for any input point $\y\in\reals^n$, it returns some $\brac{\x,\widetilde{\y}}\in\mK\times\reals^n$ such that i.
for all $\z\in\mK$, $\Vert \widetilde{\y} - \z \Vert \leq \Vert \y - \z \Vert $, and ii.
$\Vert{\x-\widetilde{\y}}\Vert^2 \leq \epsilon$.
\end{definition}

Algorithm \ref{alg:CIP-FW} given below is taken from \cite{Garber23a} and implements an AFP oracle for the set $\mK$ using only linear optimization steps over $\mK$. It uses as a sub-routine the well known Frank-Wolfe method (Algorithm \ref{alg:SH-FW}).
 If the point-to-project is far from the set, the output of Algorithm \ref{alg:SH-FW} is used to construct a hyperplane separating the point from $\mK$, which  is then used to ``pull the point closer'' to $\mK$. Otherwise, if the point is already sufficiently close to $\mK$ (but not necessarily in $\mK$),  Algorithm \ref{alg:SH-FW} outputs a point in $\mK$ sufficiently close to it.

\begin{algorithm}
\begin{algorithmic}
\STATE\textbf{Data: } parameters $\y_1\in\reals^n,\x_0\in\mK, \epsilon$
    \IF{$\Vert \x_{0} - \y_{1} \Vert^2 \leq 3\epsilon$}
        \STATE \textbf{return} $\x \gets \x_{0}$, $\y \gets \y_{1}$
    \ENDIF
    \FOR{$i=1,2, \dots$}
    \STATE $\x_{i} \gets$ Output of Algorithm \ref{alg:SH-FW} when called with tolerance $\epsilon$, feasible point $\x_{i-1}$, and  initial point $\y_{i}$\\
    \IF{$\Vert \x_{i} - \y_{i} \Vert^2 > 3\epsilon$}
      \STATE  $\y_{i+1} =  \y_{i} - \frac{2}{3}\left( \y_{i} - \x_{i} \right)$
    \ELSE 
    \STATE \textbf{return} $\x \gets  \x_{i}$, $\y \gets \y_i$
    \ENDIF
    \ENDFOR
\end{algorithmic}
 \caption{Approximately-Feasible Projection via  Linear Optimization Oracle (see \cite{Garber23a})}\label{alg:CIP-FW}
\end{algorithm}
\begin{algorithm}
\begin{algorithmic}
  \STATE \textbf{Data: } {parameters $\x_1\in\mK$, $\y\in\reals^n$, $\epsilon$}
  \FOR{ $i =1,2, \dots$}
        \STATE $ \mathbf{v}_{i} \in \textrm{argmin}_{\x \in \mK} \{ (\x_{i} - \y)^{\top} \x \} $ \COMMENT{call to LOO of $\mK$}
        \IF{$( \x_i - \y )^\top (\x_i -\v_i) \leq \epsilon$ or $\Vert \x_{i} - \y \Vert^2 \leq 3\epsilon$}
	        \STATE \textbf{return} $\widetilde{\x} \gets \x_{i}$
	        \ENDIF
	\STATE    $ \sigma_{i} = \textrm{argmin}_{\sigma \in [0, 1]}  \{ \Vert \y - \x_{i} - \sigma (\mathbf{v}_i - \x_{i})) \Vert^2 \}$
	\STATE $\x_{i+1} = \x_i + \sigma_{i} (\mathbf{v}_i - \x_i) $
    \ENDFOR
\end{algorithmic}
  \caption{Separating Hyperplane via Frank-Wolfe}\label{alg:SH-FW}
\end{algorithm}

The following lemma is  from \cite{Garber23a}.
\begin{lemma} \label{lemma:CIP-FW}
Let $\mK\subset\reals^n$ be convex and compact such that $\mathbf{0}\in\mK\subseteq{}R\ball$ for some $R>0$.
Algorithm \ref{alg:CIP-FW} guarantees that it 
returns $(\x,\y) \in \mK\times \reals^n$ such that the following three conditions hold: 
   I. $\forall \z \in \mK :$ $\enorm{ \y - \z }^2 \leq  \enorm{ \y_{1} - \z }^2$, 
    II. $\enorm{ \x - \y }^2 \leq 3\epsilon$, and III. $\Vert{\y}\Vert \leq \Vert{\y_1}\Vert$.
The overall number of calls to the LOO of $\mK$ (in Algorithm \ref{alg:SH-FW}) throughout the run of Algorithm \ref{alg:CIP-FW} is upper-bounded by
$\frac{27 R^2}{\epsilon}\max \left\{2.25\log\brac{ \frac{\enorm{ \y_{1} -\x_{0}}^2 }{ \epsilon} }+1, 0 \right\}$. 
\end{lemma}




\section{Drift-plus-Penalty-inspired  Algorithm with Adaptive Regret Guarantees}\label{sec:d+p}
Our first online algorithm, Algorithm \ref{alg:d+p}, is inspired by the \textit{drift plus penalty} approach used extensively in the literature on OCO with constraints, see for instance \cite{NEURIPS2022_ec360cb7, yi2021regret, yi2022regret, JMLR:v21:16-494}. The algorithm considers the $T$ rounds  in disjoint blocks of fixed size $K$, where $K$ is a parameter of the algorithm. At the end of each block $m\in[T/K]$, the algorithm performs its two key steps: The first step is to find a minimizer of a  \textit{drift-plus-penalty}-style objective function (the function $h_m$ defined in the algorithm), involving the gradients (i.e., linearizations) of the loss functions  and the constraints functions (without linearization) observed during the $K$ iterations of the block, over a Euclidean ball enclosing the set $\mK$ (note this operation does not involve any optimization over the set $\mK$)\footnote{We note that our \textit{drift-plus-penalty}-style objectives are simpler versions of those used in previous works such as  \cite{NEURIPS2022_ec360cb7, yi2021regret, yi2022regret, JMLR:v21:16-494} since the penalty multiplying the constraint function $G^+_m$ is fixed throughout the run.}. The second step is to feed to computed minimizer into the approximately-feasible projection oracle --- Algorithm \ref{alg:CIP-FW}, to obtain the next state of the algorithm.
\begin{algorithm}
\begin{algorithmic}
\STATE \textbf{Data: }{parameters $T$, $K$, $\epsilon$, $\delta$, $\alpha$}
\STATE $\x_1=\widetilde{\y}_1 \gets $ arbitrary point in $\mK$
\FOR{$~ m = 1,\ldots,T/K ~$}
\FOR{$~ t = (m-1)K+1,\ldots,mK ~$}
    \STATE Play $\x_{m} $ and observe $f_{t}(\cdot), g_t(\cdot)$ 
    \STATE Set $\nabla_{t} \in \partial f_t(\widetilde{\y}_{m})$
\ENDFOR
    \STATE Denote $\mathcal{T}_m = \{ (m-1)K +1 , \dots, mK \}$
    \STATE $\bar{\nabla}_m \gets \frac{1}{K} \sum_{t=\mathcal{T}_m} \nabla_t$
    \STATE $G_m^+(\x) \gets \delta \sum_{t\in\mathcal{T}_m} g_t^+(\x)$
    \STATE  $\y_{m+1} \gets \textrm{argmin}_{\x \in R\ball}    \{h_m(\x) := \bar{\nabla}_m^\top \brac{\x -\widetilde{\y}_m} + G_m^+(\x) + \frac{\alpha}{2} \enorm{\x - \widetilde{\y}_{m}}^2\}$
    \STATE  $\brac{\x_{m+1},\widetilde{\y}_{m+1}}\gets \mathcal{O}_{AFP}\brac{\y_{m+1},\x_m,\epsilon,\mK}$ \COMMENT{Alg. \ref{alg:CIP-FW}}
\ENDFOR
\end{algorithmic}
\caption{LOO-based Drift-plus-Penalty Method } \label{alg:d+p}
\end{algorithm}


\begin{theorem}\label{thm:d+p}   
Consider running Algorithm \ref{alg:d+p} with parameters $\delta = K = T^\frac{1}{2}$ $, \alpha = 
\frac{G_f}{R} T^\frac{1}{4}, \epsilon = 61 R^2 T^{-\frac{1}{2}} \log{ T } $. Fix an interval $[s,e], 1\leq s < e \leq T$. If $\widehat{\mK}_{s,e} := \{ 
\x \in \mK : g_t(\x) \leq 0, \forall t\in [s,e]\} \neq \emptyset $ then the regret  on the interval is upper-bounded by
\begin{align*}
        \sum_{t=s}^{e} f_t(\x_{m(t)}) - \min_{\x \in \widehat{\mK}_{s,e}} \sum_{t=s}^{e} f_t(\x) = O\left({R G_f T^\frac{3}{4} \sqrt{\log{ T }}}\right),
\end{align*}
and the constraints violation on the interval is upper bounded by
    $\sum_{t=s}^{e} g_t^+(\x_{m(t)})  =  O\left({RG_gT^\frac{7}{8} }\right)$.
The overall number of calls to the LOO of $\mK$ is upper-bounded by $T$. 
\end{theorem}
Note Theorem \ref{thm:d+p} in particular provides meaningful bounds  in the important scenario in which no single point in $\mK$ satsfies all the soft constraints given by $g_1,\dots,g_T$.

Relying only on a LOO (with an overall budget of $T$ calls) comes with a price: the bounds in Theorem \ref{thm:d+p} could be compared to the $O(\sqrt{T})$ regret and $O(T^{3/4})$ constraint violation achievable by methods with \textit{unrestricted optimizations}, see for instance \cite{NEURIPS2022_ec360cb7}. This deterioration in rates is expected and is a  known phenomena in projection-free methods for OCO (without soft constraints), e.g., \cite{Hazan12, Garber22a}. 

For  the proof of the theorem we need the following  lemma.
\begin{lemma}\label{lemma:d+p_afp_regret}
   Consider the (infeasible) sequence  $\{\widetilde{\y}_{m(t)}\}_{t\in[T]}$ produced by Algorithm \ref{alg:d+p}. Fix an interval
    $[s,e], 1\leq s < e \leq T$. If $\widehat{\mK}_{s,e}  \neq \emptyset$ then the regret of   $\{\widetilde{\y}_{m(t)}\}_{t\in[T]}$  is upper-bounded by:
    \begin{align*}
        \hspace{-3pt}\sum_{t=s}^{e} \hspace{-1pt}f_t(\widetilde{\y}_{m(t)}) \hspace{-1pt}-   \hspace{-5pt}\min_{\x \in \widehat{\mK}_{s,e}} \hspace{-2pt}\sum_{t=s}^{e}\hspace{-1pt} f_t(\x) \hspace{-1pt}\leq \hspace{-1pt} 2 KR(2G_f  \hspace{-1pt} +\hspace{-1pt}  R \alpha) \hspace{-2pt}+\hspace{-2pt}  \frac{ G_f^2}{2 \alpha}T,
    \end{align*}
    and for every $c>0$,  the constraints violation of the sequence on the interval is upper-bounded by:
    \begin{align*}
        \sum_{t=s}^{e} g_t^+(\widetilde{\y}_{m(t)})  &\leq \frac{G_g^2 }{2 c} T + 2cR \brac{\frac{ G_f T }{\alpha} + R K}  \\
        &+ \frac{2 \alpha R^2}{\delta}  +   \brac{ \frac{ G_f }{ \alpha }  +  4 R  } \frac{G_f T}{2 \delta K} + 4KRG_g.
    \end{align*}
\end{lemma}

\begin{proof}
    Fix some $m \in [T/K]$. Note $h_m(\cdot)$ is  $\alpha$-strongly convex. Since $\y_{m+1} \in R\ball$ minimizes $h_m(\x)$ over $R\ball$, if follows that for any $\x\in{}R\ball$, $ h_m(\x)  \geq h_m(\y_{m+1}) + \frac{\alpha}{2} \enorm{\x - \y_{m+1}}^2$. Thus, for every $\x \in \mK\subseteq{}R\ball$,
    \begin{align*}
        &\bar{\nabla}_m^\top \brac{\x - \widetilde{\y}_m}\hspace{-1pt}   +\hspace{-1pt}  G_m^+(\x)\hspace{-1pt}  + \hspace{-1pt} \frac{\alpha}{2}\hspace{-1pt}   \brac{ \enorm{\x - \widetilde{\y}_{m}}^2 \hspace{-1pt}-\hspace{-1pt} \enorm{\x - \y_{m+1}}^2 } \nonumber \\
        & \geq \bar{\nabla}_m^\top  \brac{ \y_{m+1} - \widetilde{\y}_m} + G_m^+( \y_{m+1}) + \frac{\alpha}{2} \enorm{ \y_{m+1} - \widetilde{\y}_{m}}^2. 
    \end{align*}
    Since $\widetilde{\y}_{m+1}$ is the output of $\mathcal{O}_{AFP}\brac{\y_{m+1},\epsilon,\mK}$, it follows that  $\enorm{\x - \widetilde{\y}_{m+1}} \leq \enorm{\x - \y_{m+1}}$ (Lemma \ref{lemma:CIP-FW}) and so, 
    \begin{align}
        &\bar{\nabla}_m^\top \brac{\x - \widetilde{\y}_m} \hspace{-1pt} +\hspace{-1pt}  G_m^+(\x)  \hspace{-1pt} + \hspace{-1pt} \frac{\alpha}{2}\hspace{-1pt}   \brac{\hspace{-1pt}  \enorm{\x - \widetilde{\y}_{m}}^2 \hspace{-1pt} -\hspace{-1pt}  \enorm{\x - \widetilde{\y}_{m+1}}^2 }  \nonumber \\
        & \geq  \bar{\nabla}_m^\top  \brac{\y_{m+1} - \widetilde{\y}_m} + G_m^+(\y_{m+1}) + \frac{\alpha}{2} \enorm{\y_{m+1} - \widetilde{\y}_{m}}^2. \label{eq:h_optimal_leq_every_x}
    \end{align}
Note that
    \begin{align*}
        &\bar{\nabla}_m^\top  \brac{\y_{m+1} -\widetilde{\y}_m} + \frac{\alpha}{2} \enorm{\y_{m+1} - \widetilde{\y}_{m}}^2 \\
        &= \enorm{\frac{\bar{\nabla}_m}{\sqrt{2 \alpha}} + \sqrt{\frac{\alpha}{2}} \brac{\y_{m+1} - \widetilde{\y}_{m}} }^2  \hspace{-2pt} -\hspace{-1pt}   \frac{\enorm{\bar{\nabla}_m}^2}{2 \alpha}\hspace{-1pt}  \geq  \hspace{-1pt} -  \frac{\enorm{\bar{\nabla}_m}^2}{2 \alpha}.
    \end{align*}
    Combining the last two inequalities, we have for any $\x\in\mK$,
    \begin{align*}
        &\bar{\nabla}_m^\top \brac{\x - \widetilde{\y}_m} + G_m^+(\x)  + \frac{\alpha}{2}\brac{ \enorm{\x - \widetilde{\y}_{m}}^2 - \enorm{\x - \widetilde{\y}_{m+1}}^2 } \\
        &\geq   G_m^+(\y_{m+1}) -  \frac{ \enorm{\bar{\nabla}_m}^2}{2 \alpha},
    \end{align*}
    which by rearranging gives that for every $\x \in \mK$,
    \begin{align}
        \bar{\nabla}_m^\top \brac{\widetilde{\y}_m - \x} &\leq   \frac{\alpha}{2}\brac{ \enorm{\x - \widetilde{\y}_{m}}^2 - \enorm{\x - \widetilde{\y}_{m+1}}^2 } \nonumber \\
        & + \frac{ \enorm{\bar{\nabla}_m}^2}{2 \alpha} +  G_m^+(\x) - G_m^+(\y_{m+1}). \label{eq:d+p_gradient_leq_than} 
    \end{align}
    
    Let us now fix some interval $[s,e], 1\leq s \leq e\leq T$. Let $m_s$ and $m_e$ denotes the smallest block index and the largest block index that are fully contained in the interval $[s,e]$, respectively. Note that $\enorm{\bar{\nabla}_m} \leq G_f$ for every $m$. Recall $G_m^+(\cdot) \geq 0$,  and  $G_m^+(\x) = 0$ for every $\x \in \widehat{\mK}_{s,e}$ and $m\in[m_s,m_e]$. Summing Eq.\eqref{eq:d+p_gradient_leq_than} over $m$ from $m_s$ to $m_e$, we have that for every $\x \in \widehat{\mK}_{s,e}$ it holds that,
    \begin{align}
        \hspace{-10pt}\sum_{m=m_s}^{m_e} \sum_{t\in \mathcal{T}_m}\hspace{-2pt}  \nabla_t^\top \brac{\widetilde{\y}_{m} - \x} \leq         K \frac{\alpha}{2} \enorm{\x - \widetilde{\y}_{m_s}}^2 +  \frac{ G_f^2}{2 \alpha} 
         T. \label{eq:regret_full_blocks}
    \end{align}
    Since $f_t(\cdot)$ is convex, for every $\x \in \widehat{\mK}_{s,e}$ it holds that 
    \begin{align*}
        &\sum_{t=s}^{e} f_t(\widetilde{\y}_{m(t)}) - f_t(\x) =  \sum_{t=s}^{(m_s-1)K} f_t(\widetilde{\y}_{m_s - 1}) - f_t(\x) \\
        &+ \hspace{-5pt}\sum_{m=m_s}^{m_e} \sum_{t\in \mathcal{T}_m} f_t(\widetilde{\y}_{m}) - f_t(\x)  +\hspace{-10pt} \sum_{t=m_eK+1}^{e} \hspace{-10pt}f_t(\widetilde{\y}_{m_e + 1}) - f_t(\x) \\ 
        &\leq  \sum_{t=s}^{(m_s-1)K} \enorm{\nabla_t} \enorm{\widetilde{\y}_{m_s - 1} - \x} + \hspace{-6pt}\sum_{m=m_s}^{m_e} \sum_{t\in \mathcal{T}_m} \nabla_t^\top \brac{\widetilde{\y}_{m} - \x} \\
        & ~+ \sum_{t=m_eK+1}^{e} \enorm{\nabla_t} \enorm{\widetilde{\y}_{m_e + 1} - \x}.
    \end{align*}
    From Lemma \ref{lemma:CIP-FW} it follows that  $\widetilde{\y}_m \in R\ball$ for every $m$. plugging this, Eq.\eqref{eq:regret_full_blocks}, and the fact that $\forall t\in[T]$ $\Vert{\nabla_t}\Vert \leq G_f$   into the last inequality, we have that for every $\x \in \widehat{\mK}_{s,e}$,
        \begin{align*}
        \sum_{t=s}^{e} f_t(\widetilde{\y}_{m(t)}) - f_t(\x) \leq & 4 K G_f R + 2 K R^2 \alpha +  \frac{ G_f^2}{2 \alpha}T.
    \end{align*}
    Now, we move on to bound the  constraint violation on the interval $[s,e]$. 
    We start with some preliminaries.
    Rearranging Eq.\eqref{eq:h_optimal_leq_every_x}, we have  for every $\x \in \mK$ and $m \in [T/K]$ that, 
    \begin{align*}
         &\enorm{\y_{m+1} - \widetilde{\y}_{m}}^2  \leq \frac{2}{\alpha}  \bar{\nabla}_m^\top \brac{\x - \y_{m+1}}\\
         &\hspace{-1pt} + \hspace{-1pt}\frac{2}{\alpha}  \brac{G_m^+(\x) \hspace{-1pt}- \hspace{-1pt}G_m^+(\y_{m+1})}  \hspace{-1pt}+ \hspace{-1pt} \enorm{\x - \widetilde{\y}_{m}}^2 \hspace{-1pt}-\hspace{-1pt} \enorm{\x - \widetilde{\y}_{m+1}}^2 \\
         & \leq \frac{4R}{\alpha}  \enorm{\bar{\nabla}_m} \hspace{-1pt} +\hspace{-1pt} \frac{2}{\alpha}  G_m^+(\x)  \hspace{-1pt}+\hspace{-1pt}   \enorm{\x - \widetilde{\y}_{m}}^2 \hspace{-1pt}-\hspace{-1pt} \enorm{\x - \widetilde{\y}_{m+1}}^2,
    \end{align*}    
    where in the last inequality we have used the facts that $\x\in\mK\subseteq{}R\ball$ and $\y_{m+1}\in{}R\ball$, and that $G_m^+(\cdot) \geq 0$ .
    Summing the above inequality over $m$ from $m_s$ to $m_e$, and recalling that $\widetilde{\y}_m \in R\ball$  and $\enorm{\bar{\nabla}_m} \leq G_f$ for all $m$, we have that  for every $\x \in \mK$ it holds that, 
    \begin{align}
        \hspace{-12pt}\sum_{m=m_s}^{m_e}\hspace{-5pt} \enorm{\y_{m+1} - \widetilde{\y}_{m}}^2 &  \hspace{-1pt}\leq\hspace{-1pt}  \frac{ 4 G_f R T }{\alpha K}   \hspace{-2pt}+\hspace{-2pt} \frac{2}{\alpha} \hspace{-3pt}\sum_{m=m_s}^{m_e}\hspace{-6pt} G_m^+(\x) \hspace{-2pt}+\hspace{-2pt}  4R^2. \label{eq:sum_of_steps_dist}
    \end{align}
    Rearranging Eq.\eqref{eq:d+p_gradient_leq_than}, and summing it over $m$ from $m_s$ to $m_e$, we have that for every $\x \in \mK$ it holds that, 
    \begin{align*}
        \hspace{-5pt}\sum_{m=m_s}^{m_e} \hspace{-5pt}G_m^+(\y_{m+1}) & \leq  \hspace{-5pt} \sum_{m=m_s}^{m_e} \hspace{-3pt}\frac{\alpha}{2}  \brac{ \enorm{\x - \widetilde{\y}_{m}}^2 - \enorm{\x - \widetilde{\y}_{m+1}}^2 }\\
        &+\hspace{-6pt}\sum_{m=m_s}^{m_e} \hspace{-5pt}  \left({\frac{ \hspace{-2pt}\enorm{\bar{\nabla}_m}^2}{2 \alpha}  \hspace{-1pt}+ \hspace{-1pt}  G_m^+(\x)  \hspace{-1pt}+ \hspace{-1pt}  \bar{\nabla}_m^\top \brac{\x - \widetilde{\y}_m }\hspace{-2pt}}\right) \\
        & \leq  2 \alpha R^2 +   \frac{ G_f^2 T}{2 \alpha K} +  \hspace{-6pt}\sum_{m=m_s}^{m_e} \hspace{-5pt}G_m^+(\x) + \frac{2G_fRT}{K}. 
    \end{align*}
    Recalling  $G_m^+(\x) = \delta  \sum_{t\in \mathcal{T}_m}g_t^+(\x)$, 
    we have $\forall\x\in\mK$, 
    \begin{align}
        \hspace{-14pt}\sum_{m=m_s}^{m_e}\sum_{t\in\mathcal{T}_m} g_t^+(\y_{m(t)+1}) &\leq  \frac{2 \alpha R^2}{\delta}  +  \frac{ G_f^2 T}{2 \alpha \delta K} \nonumber \\
        &+  \frac{1}{\delta} \sum_{m=m_s}^{m_e} G_m^+(\x) + \frac{ 2 G_f R T}{\delta K}. \label{eq:sum_violations_playing_y_tilde}
    \end{align}
   Since $g_t(\cdot)$ is convex and $G_g$-Lipschitz over $R\ball$ for every $t$,  $g_t^+(\x) = \max\{0, g_t(\x)\}$ is also convex and $G_g$-Lipschitz over $R\ball$. Using the inequality $2ab \leq a^2 + b^2$ for every $a,b \in \reals$, we have that for every $t\in[T]$ and $c>0$,
    \begin{align*}
        g_t^+(\widetilde{\y}_{m(t)}) - g_t^+(\y_{{m(t)+1}}) & \leq G_g  \enorm{\widetilde{\y}_{m(t)} - \y_{m(t)+1}} \\
        &\leq  \frac{G_g^2}{2 c} + \frac{c}{2} \enorm{ \widetilde{\y}_{m(t)} - \y_{m(t)+1} }^2. 
    \end{align*}
    Summing the last inequality over $t$, and combining it with Eq.\eqref{eq:sum_of_steps_dist}, and Eq.\eqref{eq:sum_violations_playing_y_tilde}, we have that  for every $\x \in \widehat{\mK}_{s,e} \subseteq \mK$,  
    \begin{align*}
        &\hspace{-4pt}\sum_{m=m_s}^{m_e}\sum_{t\in\mathcal{T}_m} g_t^+(\widetilde{\y}_{m(t)})  \leq  \frac{G_g^2}{2 c}T \hspace{-1pt}+\hspace{-1pt} \frac{cK}{2}\hspace{-4pt} \sum_{m=m_s}^{m_e}\hspace{-4pt}\enorm{\widetilde{\y}_{m} - \y_{m+1}}^2 \nonumber \\
        & + \sum_{m=m_s}^{m_e}\sum_{t\in\mathcal{T}_m}g_t^+(\y_{{m(t)+1}}) \nonumber  \\
       & \leq  \frac{G_g^2 }{2 c} T + \frac{ 2 c G_f R T }{\alpha} + 2 c R^2 K + \brac{\frac{cK}{\alpha} +  \frac{1}{\delta}}\hspace{-5pt} \sum_{m=m_s}^{m_e}\hspace{-5pt} G_m^+(\x) \nonumber \\
        &+ \frac{2 \alpha R^2}{\delta}  +  \frac{ G_f^2 T}{2 \alpha \delta K}  + \frac{ 2 G_f R T}{\delta K}\\
        &= \hspace{-2pt} \frac{G_g^2T }{2 c} \hspace{-2pt}+\hspace{-2pt} \frac{ 2 c G_f R T }{\alpha} \hspace{-2pt}+ \hspace{-2pt}2 c R^2 K \hspace{-2pt}+ \hspace{-2pt}\frac{2 \alpha R^2}{\delta}  \hspace{-2pt}+ \hspace{-2pt} \frac{ G_f^2 T}{2 \alpha \delta K}  \hspace{-2pt}+\hspace{-2pt} \frac{ 2 G_f R T}{\delta K},
    \end{align*}
where that last equality holds since $G_m^+(\x) = 0 $ for every $\x \in \widehat{\mK}_{s,e}$ and $m\in[m_s,m_e]$. 

The proof is concluded by noting that the number of first  and last  iterates in the interval $[s,e]$ that are not contained in the blocks $m_s,\dots{}m_e$, i.e., the set $[s,e]\setminus\bigcup_{m=m_s}^{m_e}\mathcal{T}_m$, is at most $2K$. These add an additional term which is at most $2K\cdot2RG_g$ to the RHS of the last inequality by using the fact that for each $t\in[s,e]\setminus\bigcup_{m=m_s}^{m_e}\mathcal{T}_m$ and $\x\in\widehat{\mK}_{s,e}$,
$g_t^+(\widetilde{\y}_{m(t)}) \leq g_t^+(\x) + \Vert{\widetilde{\y}_{m(t)}-\x}\Vert\cdot{}G_g \leq 2RG_g$.
\end{proof}

\begin{proof}[Proof of Theorem \ref{thm:d+p}]
Fix some interval $[s,e], 1\leq s \leq e\leq T$ such that $\widehat{\mK}_{s,e}\neq \emptyset$. Applying Lemma \ref{lemma:d+p_afp_regret} and Lemma \ref{lemma:CIP-FW}, we have that for every $\x \in \widehat{\mK}_{s,e}$ it holds that,
\begin{align*}
    &\sum_{t=s}^{e} f_t(\x_{m(t)}) - f_t(\x)  =\nonumber \\
    &\sum_{t=s}^{e} f_t(\x_{m(t)}) - f_t(\widetilde{\y}_{m(t)}) +f_t(\widetilde{\y}_{m(t)})  - f_t(\x)\leq \nonumber\\
    &  4 K G_f R + 2 K R^2 \alpha +  \frac{ G_f^2}{2 \alpha}T + \sum_{t=s}^{e} \enorm{\nabla_t} \enorm{\x_{m(t)} - \widetilde{\y}_{m(t)}} \nonumber\\
    & \leq 4 K G_f R + 2 K R^2 \alpha +  \frac{ G_f^2}{2 \alpha}T + T G_f \sqrt{3\epsilon},
\end{align*}
where the last inequality also uses the facts that $\widetilde{\y}_m \in R\ball$ and that $\Vert{\nabla_t}\Vert\leq G_f$  for every $t\in[T]$. 

plugging-in the values of $K,\alpha,\epsilon$ stated in the theorem, we obtain the required bound on the  regret.

Now we move on to bound the constraints violation. Since for every $t\in[T]$, $g_t(\cdot)$ is convex and $G_g$-Lipschitz over $R\ball$, it follows that $g_t^+(\x)$ is also convex and $G_g$-Lipschitz over $R\ball$. Using Lemma \ref{lemma:CIP-FW} we have that for every $t\in[T]$,
\begin{align*}
    g_t^+(\x_{m(t)}) - g_t^+(\tilde{\y}_{{m(t)}}) & \leq G_g  \enorm{\x_{m(t)} - \tilde{\y}_{m(t)}} \leq G_g \sqrt{3 \epsilon} . 
\end{align*}
Summing over $t\in[s,e]$, and using Lemma \ref{lemma:d+p_afp_regret}, it holds that 
\begin{align*}
    &\sum_{t=s}^{e} g_t^+(\x_{m(t)})  \leq   \frac{G_g^2 }{2 c} T + 2cR \brac{\frac{ G_f T }{\alpha} + R K}  + \frac{2 \alpha R^2}{\delta}  \\
    &+   \brac{ \frac{ G_f }{ \alpha }  +  4 R  } \frac{G_f T}{2 \delta K} +  4KRG_g + G_g \sqrt{3 \epsilon} T.
\end{align*}
Plugging-in $c = \frac{G_g}{12 R}T^\frac{1}{8}$ and the values of $K,\alpha,\epsilon, \delta$ listed in the theorem, the overall constraints violation bound stated in the theorem is obtained. 

Finally, we move on to bound the overall number of calls to the linear optimization oracle. Using Lemma \ref{lemma:CIP-FW}, the overall number of calls to linear optimization oracle is 
{\small
\begin{align*}
    N_{calls} &\leq \sum_{m=2}^{T/K} \frac{27R^2  }{\epsilon } \brac{2.25\log\brac{ \frac{\enorm{\tilde{\y}_{m+1} - \x_m}^2 }{ \epsilon} }+1} \\
     &\leq \sum_{m=2}^{T/K} \frac{27R^2  }{\epsilon } \brac{2.25\log\brac{ \frac{4R^2 }{ \epsilon} }+2.25 \log (e^\frac{4}{9})}\\
    & \leq \frac{27R^2 T }{\epsilon K} \brac{2.25\log\brac{ \frac{7R^2 }{ \epsilon} }} \leq \frac{61R^2 T }{\epsilon K} \log\brac{ \frac{7R^2 }{ \epsilon} } .
\end{align*}}
Plugging-in the values of $\epsilon, K$ listed in the theorem, we obtain that  $N_{calls} \leq T$, as required.
\end{proof}

\section{Lagrangian-based Primal-Dual Algorithm}\label{sec:lagrange}
Our second algorithm, Algorithm \ref{alg:LF}, is based on an online first-order primal-dual approach which has also been studied extensively, see for instance \cite{mahdavi2012trading, cao2018online, pmlr-v48-jenatton16, pmlr-v70-sun17a}. This algorithm is more efficient than Algorithm \ref{alg:d+p} since it does not require any exact minimization of an optimization problem involving the constraint functions $g_t, t\in[T]$, and instead only requires first-order access (i.e., to compute subgradients) of $f_t,g_t, t\in[T]$. On the downside,  we will only have standard (i.e., non-adaptive) guarantees for this method that hold only w.r..t. the entire sequence of functions. 

The algorithm performs online primal-dual (sub)gradient updates w.r.t. the dual-regularized Lagrangian functions 
\begin{align*}
\mathcal{L}_t(\x,\lambda) := f_t(\x) + \lambda g_t^+(\x) - \frac{\delta \eta}{2} \lambda^2 \quad t\in[T]
\end{align*} 
and their block aggregations $\mathcal{L}_m(\x,\lambda) := \sum_{t=(m-1)K+1}^{mK}\mathcal{L}_t(\x,\lambda), m\in[T/K]$. Note that the vectors $\nabla_{m,\x}, \nabla_{m,\lambda}$ in Algorithm \ref{alg:LF} are indeed the derivatives of $\mathcal{L}_m(\x,\lambda)$ w.r.t. $\x$ and $\lambda$, evaluated at $(\x_m,\lambda_m)$.

\begin{algorithm}
\begin{algorithmic}
\STATE \textbf{data: }{parameters $T$, $K$, $ \epsilon$, $\delta$, $\eta$}
\STATE $\x_1=\tilde{\y}_1 \gets $ arbitrary point in $\mK$, $\lambda_1 \gets 0$
\FOR{$~ m = 1,\ldots,T/K ~$}
\FOR{$~ t = (m-1)K+1,\ldots,mK ~$}
    \STATE Play $\x_{m} $ and observe $f_{t}(\cdot), g_t(\cdot)$ 
    \STATE Set $\nabla{}f_t \in \partial f_t(\x_{m})$ and $\nabla_{}g^+_t \in \partial g_t^+(\x_{m})$
\ENDFOR
    \STATE Denote $\mathcal{T}_m = \{ (m-1)K +1 , \dots, mK \}$
    \STATE  $\nabla_{m,\x} \gets  \sum_{t\in \mathcal{T}_m} \nabla{}f_t + \lambda_m \nabla{}g^+_t$
    \STATE  $\nabla_{m,\lambda} \gets \sum_{t\in \mathcal{T}_m}\left({  g_t^+(\x_m) -\delta \eta \lambda_m}\right)$
    \STATE  $\y_{m+1} \gets \Pi_{R\ball}[\tilde{\y}_m - \eta \nabla_{m,\x}]$
    \STATE  $\brac{\x_{m+1},\tilde{\y}_{m+1}}\gets \mathcal{O}_{AFP}\brac{\y_{m+1},\x_m,\epsilon,\mK}$    \COMMENT{Alg.  \ref{alg:CIP-FW}}
    \STATE  $\lambda_{m+1} \gets\Pi_{\reals_{+}} \brac{\lambda_m  + \eta \nabla_{m,\lambda}}$
\ENDFOR
\end{algorithmic}
\caption{LOO-based Primal-Dual Method}
 \label{alg:LF}
\end{algorithm}

\begin{theorem}\label{thm:LF} 
Suppose  $\widehat{\mK} = \{ \x \in \mK :  g_t(\x) \leq 0~ \forall t\in [T] \}\neq \emptyset$. For every $T$ sufficiently large, setting $\delta = 32\brac{G_g^2 + G_g R}T^{\frac{1}{2}}\sqrt{\log{{T}}}, \eta = 
 T^{-\frac{3}{4}}, \epsilon = 61 R^2 T^{-\frac{1}{2}} \log{T } , K = T^\frac{1}{2}$ in Algorithm \ref{alg:LF} guarantees that the regret is upper bounded by
\begin{align*}
   &\sum_{t=1}^{T} f_t(\x_{m(t)}) - \min_{\x \in \widehat{\mK}}\sum_{t=1}^{T}\hspace{-1pt} f_t(\x) =\\
   & O\left({R\brac{G_f + G_g}  T^{\frac{3}{4}}\sqrt{  \log{T}} + (G_f^2+(G_g^2+1)R^2)T^{3/4}}\right),
%
\end{align*}
 the constraints violation is upper bounded by
\begin{align*}
    \hspace{-3pt}\sum_{t=1}^{T} g_t^+(\x_{m(t)}) \hspace{-2pt}=\hspace{-2pt} O\hspace{-2pt}\left({\hspace{-3pt}\sqrt{G_f R\left({1\hspace{-1pt}+\hspace{-1pt} G_g\brac{G_g + R }\sqrt{\log{T}}}\right)} T^{\frac{7}{8}}\hspace{-2pt}}\right)\hspace{-2pt},
\end{align*}
and the overall number of calls to the LOO is at most  $T$.
\end{theorem}

The proof of the theorem builds on the following lemma.

\begin{lemma}\label{lemma:LF}
Algorithm \ref{alg:LF} guarantees that for every $\x \in \widehat{\mK}$ and $\lambda \in \reals_+$, 
\begin{align*}
    &\sum_{t=1}^{T} \mathcal{L}_t \brac{\x_{m(t)},\lambda}  - \mathcal{L}_t \brac{\x,\lambda_{m(t)}}  \leq \\
    &\frac{2R^2}{\eta} + \frac{\lambda^2}{2\eta} + G_f \sqrt{3\epsilon} T  +  G_f^2  \eta K T   + 4 G_g^2 R^2 \eta K T\\
    &  + G_g \sqrt{3\epsilon} K \sum_{m=1}^{T/K} \lambda_m    + \brac{\delta^2 \eta^3 K^2  +  G_g^2 \eta K^2} \sum_{m=1}^{T/K} \lambda_m^2.
\end{align*}
\end{lemma}
\begin{proof}
Using Lemma \ref{lemma:CIP-FW} and the facts that $\y_{m+1} = \Pi_{R\ball}\brac{\tilde{\y}_m - \eta \nabla_{m,\x}}$ and $\mK \subseteq R\ball$, for every $\x \in \mK$ and $m\in[T/K]$ we have that,
\begin{align*}
    &\enorm{ \tilde{\y}_{m+1} - \x}^2 \leq  \enorm{ \y_{m+1}- \x}^2 \leq \enorm{\tilde{\y}_m -\eta \nabla_{m,\x}-\x}^2\\
    &  = \enorm{\tilde{\y}_m-\x}^2 -2\eta \nabla_{m,\x}^\top \brac{\tilde{\y}_m - \x} + \eta^2 \enorm{\nabla_{m,\x}}^2,
\end{align*}
and rearranging, we have that
\begin{align}
    \nabla_{m, \x}^\top \brac{\tilde{\y}_m - \x}  & \leq \frac{1}{2\eta} \brac{ \enorm{ \tilde{\y}_m - \x}^2 - \enorm{\tilde{\y}_{m+1} -\x}^2} \nonumber\\
    &~~~+ \frac{\eta}{2} \enorm{\nabla}_{m,\x}^2. \label{eq:bound_grad_x}
\end{align}
Since $\lambda_{m+1} = \Pi_{\reals_{+}} \brac{\lambda_m  + \eta \nabla_{m,\lambda}}$, for every $\lambda \in \reals_+$ it holds that,
\begin{align*}
    &\brac{\lambda_{m+1}-\lambda}^2 \leq \brac{ \lambda_m + \eta \nabla_{m,\lambda} - \lambda}^2 \\
    &= \brac{\lambda_{m}- \lambda}^2 + 2 \eta \nabla_{m ,\lambda} \brac{ \lambda_m- \lambda } +\eta^2 \nabla_{m,\lambda}^2,
\end{align*}
and rearranging, we have that  
\begin{align}
     \nabla_{m, \lambda} \brac{ \lambda - \lambda_m} &\leq \frac{1}{2\eta} \brac{\brac{ \lambda_{m} - \lambda}^2 -  \brac{ \lambda_{m+1} - \lambda }^2} \nonumber \\
     &~~~+\frac{\eta}{2} \nabla_{m,\lambda}^2. \label{eq:bound_grad_lambda}
\end{align}
Since $\mathcal{L}_t\brac{\cdot,\lambda}$ is convex, it holds that  
\begin{align*}
   \hspace{-2pt} \mathcal{L}_t \brac{\x_m,\lambda_m}\hspace{-1pt} -\hspace{-1pt} \mathcal{L}_t \brac{\x,\lambda_m}\hspace{-1pt} \leq \hspace{-1pt}\nabla_{\x}\mathcal{L}_t\brac{\x_m,\lambda_m}^\top \brac{\x_m \hspace{-1pt}-\hspace{-1pt} \x},
\end{align*}
and since $\mathcal{L}_t\brac{\x_m , \cdot}$ is concave, it holds that
\begin{align*}
    \hspace{-2pt}\mathcal{L}_t \brac{\x_m,\lambda} \hspace{-1pt}-\hspace{-1pt} \mathcal{L}_t \brac{\x_m,\lambda_m} \hspace{-1pt}\leq\hspace{-1pt} \nabla_{\lambda}\mathcal{L}_t\brac{\x_m,\lambda_m} \brac{\lambda\hspace{-1pt}-\hspace{-1pt}\lambda_m }.
\end{align*}
Combining the last two inequalities we have that,
\begin{align*}
    \mathcal{L}_t \brac{\x_m,\lambda}  - \mathcal{L}_t \brac{\x,\lambda_m} &\leq \nabla_{\x}\mathcal{L}_t\brac{\x_m,\lambda_m}^\top \brac{\x_m - \x} \\
    &~~~+ \nabla_{\lambda}\mathcal{L}_t\brac{\x_m,\lambda_m} \brac{\lambda-\lambda_m }.
\end{align*}
Summing the above inequality over $t\in\mathcal{T}_m $ and recalling that  $\nabla_{m,\x} =  \sum_{t\in\mathcal{T}_m} \nabla_{\x}\mathcal{L}_t\brac{\x_m,\lambda_m}$ and $\nabla_{m,\lambda} = \sum_{t\in\mathcal{T}_m} \nabla_{\lambda}\mathcal{L}_t\brac{\x_m,\lambda_m}$, we have that
\begin{align*}
    &\sum_{t\in\mathcal{T}_m} \mathcal{L}_t \brac{\x_m ,\lambda}  - \mathcal{L}_t \brac{\x,\lambda_m}  \leq \\
    &\nabla_{m ,\x}^\top \brac{\x_m - \x} + \nabla_{m ,\lambda} \brac{\lambda-\lambda_m } =\\
    &  \nabla_{m ,\x}^\top \brac{\x_m - \tilde{\y}_m} + \nabla_{m ,\x}^\top \brac{\tilde{\y}_m - \x} + \nabla_{m ,\lambda} \brac{\lambda-\lambda_m }.
\end{align*}
Combining the last inequality with Eq. \eqref{eq:bound_grad_x} and \eqref{eq:bound_grad_lambda}, we have that for every $\x \in \mK$, $\lambda \in \reals_+$ and $m$, it holds that
\begin{align*}
    &\sum_{t\in\mathcal{T}_m} \mathcal{L}_t \brac{\x_m ,\lambda}  - \mathcal{L}_t \brac{\x,\lambda_m}  \leq \\
    &\nabla_{m ,\x}^\top \brac{\x_m - \tilde{\y}_m} + \frac{1}{2\eta} \brac{ \enorm{\tilde{\y}_m - \x}^2 - \enorm{\tilde{\y}_{m+1}- \x}^2  } \\
    &\hspace{-5pt}+\hspace{-2pt} \frac{\eta}{2} \enorm{\nabla_{m,\x}}^2 \hspace{-1pt}+\hspace{-1pt} \frac{1}{2\eta}\hspace{-2pt} \brac{\brac{\lambda_{m} - \lambda}^2 \hspace{-1pt} -\hspace{-1pt}  \brac{ \lambda_{m+1} - \lambda}^2 } + \frac{\eta}{2}  \nabla_{m,\lambda}^2.
\end{align*}
Summing the above over $m$, we have for every $\x \in \mK$ and $\lambda \in \reals_+$ that
\begin{align*}
   & \sum_{m=1}^{T/K} \sum_{t\in\mathcal{T}_m} \mathcal{L}_t \brac{\x_m,\lambda}  - \mathcal{L}_t \brac{\x,\lambda_m}  \leq \\
    &\sum_{m=1}^{T/K} \enorm{\nabla_{m ,\x}} \enorm{\x_m - \tilde{\y}_m}  +  \frac{\eta}{2} \sum_{m=1}^{T/K}\brac{\enorm{\nabla_{m,\x}}^2+\nabla_{m,\lambda}^2 } \nonumber \\
    & + \frac{\enorm{\tilde{\y}_1 - \x}^2 + \brac{ \lambda_{1} - \lambda}^2}{2\eta} .
\end{align*}
Using Lemma \ref{lemma:CIP-FW}, it holds that $\enorm{\x_m -\tilde{\y}_m}^2\leq 3\epsilon$ for every $m$. Since $\y_1 \in \mK$ and $\lambda_1 = 0$, we have for every $\x \in \mK$ and $\lambda \in \reals_+$ that 
\begin{align}
    &\sum_{m=1}^{T/K} \sum_{t\in\mathcal{T}_m} \mathcal{L}_t \brac{\x_m,\lambda}  - \mathcal{L}_t \brac{\x,\lambda_m} \leq 
     \sqrt{3\epsilon} \sum_{m=1}^{T/K} \enorm{\nabla_{m ,\x}} \nonumber \\
     &+ \frac{4R^2 + \lambda^2}{2\eta} +  \frac{\eta}{2} \sum_{m=1}^{T/K}\left({\enorm{\nabla_{m,\x}}^2+\nabla_{m,\lambda}^2}\right)  .\label{eq:lagrangian_regret_1}
\end{align}
Since for every $m$
\begin{align*}
  \hspace{-3pt}  \enorm{\nabla_{m,\x}}  = \enorm{ \sum\nolimits_{t\in\mathcal{T}_m}\hspace{-5pt} \nabla{}f_t + \lambda_m \nabla_{}g^+_t} \leq  K(G_f+\lambda_m G_g),
\end{align*}
and $\nabla_{m,\lambda} = \sum_{t\in\mathcal{T}_m}( g_t^+(\x_m) -\delta \eta \lambda_m)$, we obtain
{\small
\begin{align}\label{eq:lagrange:lem:1}
    &\sum_{t=1}^{T} \mathcal{L}_t \brac{\x_{m(t)},\lambda}  - \mathcal{L}_t \brac{\x,\lambda_{m(t)}} \leq \nonumber \\
    & G_g \sqrt{3\epsilon} K \sum_{m=1}^{T/K} \lambda_m   + G_f \sqrt{3\epsilon} T +\frac{4R^2 + \lambda^2}{2\eta} \nonumber \\ 
    &  +\hspace{-2pt}\frac{\eta{}K^2}{2}\hspace{-1pt} \sum_{m=1}^{T/K}(G_f\hspace{-1pt}+\hspace{-1pt}\lambda_m G_g)^2 \hspace{-1pt}+\hspace{-1pt}\frac{\eta}{2} \sum_{m=1}^{T/K}\hspace{-3pt} \left({\sum_{t\in\mathcal{T}_m} \hspace{-3pt}(g_t^+(\x_m) -\delta \eta \lambda_m)}\right)^2  \nonumber\\
    &\leq \frac{2R^2}{\eta} + \frac{\lambda^2}{2\eta} + G_f \sqrt{3\epsilon} T  +  G_f^2  \eta K T   \nonumber \\
    &+ \eta \sum_{m=1}^{T/K} \brac{\sum_{t\in\mathcal{T}_m} g_t^+(\x_m) }^2 + G_g \sqrt{3\epsilon} K \sum_{m=1}^{T/K} \lambda_m\nonumber\\
    &    + \brac{\delta^2 \eta^3 K^2  +  G_g^2 \eta K^2} \sum_{m=1}^{T/K} \lambda_m^2,
\end{align}}
where in the last inequality we have used the inequality $(a + b)^2 \leq 2a^2 +2b^2$ for every $a,b \in \reals$. 

Note that $g_t^+(\x) = 0 $ for every $\x \in \widehat{\mK}$. Since $g_t(\x)$ is convex and $G_g$-Lipschitz, $g_t^+(\x) = \max\{ 0 , g_t(\x)\}$ is also convex and $G_g$-Lipschitz and thus, for every $\x \in \widehat{\mK}$ it holds that $g_t^+(\x_m) \leq g_t^+(\x) + \nabla{}g_t^{+\top} (\x_m - \x)  \leq 2 G_g R $. Plugging this into the RHS of \eqref{eq:lagrange:lem:1} yields the lemma.

\end{proof}

\begin{proof}[Proof of Theorem \ref{thm:LF}]
From Lemma \ref{lemma:LF}, for every $\x \in \widehat{\mK}$ and $\lambda \in \reals_+$ it holds that,
\begin{align*}
    &\sum_{t=1}^{T} \brac{f_t(\x_{m(t)}) -f_t(\x)} + \sum_{t=1}^{T} \brac{\lambda g_t^+(\x_{m(t)}) - \lambda_{m(t)} g_t^+(\x) } \\
    &+ \frac{\delta \eta K }{2} \sum_{m=1}^{T/K} \lambda_m^2 - \frac{\delta \eta }{2} T \lambda^2   \leq\\ 
    &   \frac{2R^2}{\eta} + \frac{\lambda^2}{2\eta} + G_f \sqrt{3\epsilon} T  +  \brac{G_f^2 + 4 G_g^2 R^2} \eta K T \\
    &+ G_g \sqrt{3\epsilon} K \sum_{m=1}^{T/K} \lambda_m    + \brac{\delta^2 \eta^2  +  G_g^2 } \eta K^2 \sum_{m=1}^{T/K} \lambda_m^2 .
\end{align*}
Rearranging, we have that
{\small
\begin{align*}
    & \sum_{t=1}^{T} \brac{f_t(\x_{m(t)}) -f_t(\x)} + \sum_{t=1}^{T} \brac{\lambda g_t^+(\x_{m(t)}) - \lambda_{m(t)} g_t^+(\x) }\\
    & - \brac{\frac{\delta \eta }{2} T + \frac{1}{2\eta}} \lambda^2  \leq \\ 
    &   \frac{2R^2}{\eta}  + G_f \sqrt{3\epsilon} T  +  \brac{G_f^2 + 4G_g^2 R^2} \eta K T + G_g \sqrt{3\epsilon} K \sum_{m=1}^{T/K} \lambda_m   \\
    &   + \brac{\delta^2 \eta^3 K^2  +  G_g^2 \eta K^2- \frac{\delta \eta K }{2} } \sum_{m=1}^{T/K} \lambda_m^2 .
\end{align*}}
$g_t^+(\x) = 0 $ for every $\x \in \widehat{\mK}$ and $t\in[T]$. Also, by the choice of parameters in the theorem, we have that for any $T$  larger than some constant, $\delta^2 \eta^3 K^2 +  \eta K^2 G_g^2 + G_g \sqrt{3\epsilon} K \leq  \frac{\eta K\delta}{2}$. 
Thus, for every $\x \in \widehat{\mK}$ and $\lambda \in \reals_+$ we have that,
\begin{align*}
    &\sum_{t=1}^{T} \brac{f_t(\x_{m(t)}) -f_t(\x)} + \lambda \sum_{t=1}^{T}  g_t^+(\x_{m(t)})  \\
    & - \brac{\frac{\delta \eta }{2} T  + \frac{1}{2\eta}} \lambda^2 \leq  \frac{2R^2}{\eta}  + G_f \sqrt{3\epsilon} T     \\
     &+  \brac{G_f^2 + 4G_g^2 R^2} \eta K T+ G_g \sqrt{3\epsilon} K \sum_{m=1}^{T/K} \brac{ \lambda_m - \lambda_m^2} .
\end{align*}
Since $z-z^2 \leq \frac{1}{4}$ for every $z\in\reals$, it holds that
\begin{align*}
    &\sum_{t=1}^{T} \brac{f_t(\x_{m(t)}) -f_t(\x)} + \lambda \sum_{t=1}^{T}  g_t^+(\x_{m(t)})  \\
    &-  \brac{\frac{\delta \eta }{2} T  + \frac{1}{2\eta}}  \lambda^2 \leq  \\ 
     & \frac{2R^2}{\eta}  + \brac{G_f^2 + 4G_g^2 R^2} \eta K T   + \sqrt{3\epsilon} \brac{G_f + \frac{G_g}{4}}T.
\end{align*}
Let $\x^* \in \textrm{argmin}_{\x \in \widehat{\mK} }\sum_{t=1}^{T} f_t(\x)$. Note that $\lambda^* = \frac{ \sum_{t=1}^{T}  g_t^+(\x_{m(t)})}{\delta \eta  T  + \eta^{-1}}$ maximizes the term $\lambda\sum_{t=1}^{T}  g_t^+(\x_{m(t)}) - \brac{\frac{\delta \eta }{2} T  + \frac{1}{2\eta}} \lambda^2$ in the last inequality. Plugging-in $\x^*$ and $\lambda^*$ into the last inequality, we have that
\begin{align}\label{eq:lagrange:thm:1}
    &\hspace{-6pt}\sum_{t=1}^{T} \brac{f_t(\x_{m(t)}) -f_t(\x^*)} + \frac{1}{2} \frac{ \brac{\sum_{t=1}^{T}  g_t^+(\x_{m(t)})}^2}{\delta \eta  T  + \eta^{-1}} \leq \nonumber \\
    &\hspace{-6pt}\frac{2R^2}{\eta}  \hspace{-2pt}+\hspace{-2pt} \brac{G_f^2 + 4G_g^2 R^2} \eta K T  \hspace{-2pt} +\hspace{-2pt} \sqrt{3\epsilon} \brac{G_f + \frac{G_g}{4}}T.
\end{align}
This yields the regret bound
\begin{align*}
    &\sum_{t=1}^{T} f_t(\x_{m(t)}) -f_t(\x^*) \leq \\
    & \frac{2R^2}{\eta}  + \brac{G_f^2 + 4G_g^2 R^2} \eta K T   + \sqrt{3\epsilon} \brac{G_f + G_g/4}T.
\end{align*}
Plugging-in the choice of parameters listed in the theorem yields the desired regret bound.

Since $f_t(\x^*) - f_t(\x_{m(t)}) \leq 2RG_f$ for all $t\in[T]$, from \eqref{eq:lagrange:thm:1} we also obtain the constraints violation bound
{\small
\begin{align*}
    &\hspace{-6pt}\sum_{t=1}^{T}  g_t^+(\x_{m(t)}) \hspace{-2pt}\leq \hspace{-2pt}\sqrt{\delta \eta  T  + \eta^{-1}}  \bigg({\hspace{-2pt}\frac{4R^2}{\eta}  \hspace{-2pt}+\hspace{-2pt} 2\brac{G_f^2 \hspace{-2pt}+\hspace{-2pt} 4G_g^2 R^2} \hspace{-2pt}\eta K T }\bigg. \\
    &~~ \Bigg.{  + 4RG_fT  + 2\sqrt{3\epsilon} \brac{G_f + G_g/4}T }\bigg)^{1/2}.
\end{align*}}
Plugging-in our choice of parameters, yields the  bound.

Finally, the bound on the overall number of calls to the LOO of $\mK$ follows from exactly the same argument as in the proof of Theorem \ref{thm:d+p}. In particular we use the same values of $K,\epsilon$ which go into the bound.

\end{proof}

\section{Bandit Feedback Setting}\label{sec:lagrange:bandit}
In the bandit setting, after each round $t\in[T]$, the algorithm does not get to observe the functions $f_t, g_t$ (or their subgradients), but only their value at the point played on round $t$. Building on the standard approach pioneered in \cite{flaxman2005online} (see also \cite{HazanBook, garber2020improved, Garber22a}, we extend Algorithm \ref{alg:LF} to this setting by replacing the functions $f_t,g_t^+$ with their $\mu$-smoothed versions: $\widehat{f}_t(\x) = \E_{\u\sim{}U(\mS^n)}[f_t(\x+\mu\u)]$, $\widehat{g}_t^+(\x) = \E_{\u\sim{}U(\mS^n)}[g_t^+(\x+\mu\u)]$, where $U(\mS^n)$ denotes the uniform distribution over the origin-centered unit sphere. By sampling uniformly around the point to play, unbiased estimators for the gradients of $\widehat{f}_t, \widehat{g}^+_t$ \footnote{$\widehat{f}_t, \widehat{g}_t^+$ are always differentiable, even if $f_t,g_t^+$ are not} could be constructed using only the values of the functions at the sampled point.

We make the  additional standard assumptions: \textbf{I.}  $\{f_t,g_t\}_{t\in[T]}$ are chosen \textit{obliviously}, i.e., before the first round and are thus independent of any randomness introduced by the algorithm. \textbf{II.}  there is $r>0$ such that $r\ball\subseteq{}\mK$.

The algorithm, as well as the fully-detailed version of the following theorem and its proof are given in the appendix.

\begin{theorem}[short version]\label{thm:bandit:short}
Suppose  $\widehat{\mK} = \{ \x \in \mK :  g_t(\x) \leq 0~ \forall t\in [T] \}\neq \emptyset$. For any $T$ sufficiently large, there exist a choice for the parameters $\delta, \eta, \epsilon, K, \mu$ in Algorithm \ref{alg:LFbandit}, such that it guarantees expected regret: 
\begin{align*}
    \E\left[{\sum_{t=1}^{T} f_t(\z_{t})}\right] - \min_{\x \in \widehat{\mK}} \sum_{t=1}^{T} f_t(\x) =O\left({\sqrt{n}T^{3/4}\sqrt{\log{T}}}\right),
\end{align*}
 the expected constraints violation is upper bounded by
\begin{align*}
 \E\left[{\sum_{t=1}^{T}  g_t^+(\z_t)}\right] = O\left({n^{1/4}T^{7/8}{\log(T)^{1/4}}}\right),
 \end{align*}
and the overall number of calls to the LOO is at most $T$.
\end{theorem}

\bibliography{bib}
\bibliographystyle{icml2024}

\newpage
\appendix
\onecolumn
\section{Bandit Feedback Setting}
For a fixed parameter $\mu\in(0,r]$ we denote a shrinking of the set $\mK$ by $\mK_{\mu/r} := (1-\mu/r)\mK = \{(1-\mu/r)\x~|~\x\in\mK\}$.  It is well known that for all $\x\in\mK_{\mu/r}$ and $\y\in{}r\ball$, $\x+\y\in\mK$.

\begin{algorithm}
\begin{algorithmic}
\STATE \textbf{Data: }{parameters $T$, $K$, $ \epsilon$, $\delta$, $\eta$}
\STATE $\x_1=\tilde{\y}_1 \gets $ arbitrary point in $\mK_{\mu/r}$, $\lambda_1 \gets 0$
\FOR{$~ m = 1,\ldots,T/K ~$}
\FOR{$~ t = (m-1)K+1,\ldots,mK ~$}
    \STATE  $\z_t \gets \x_m + \mu\u_t$ where $\u_t\sim{}U(\mS^n)$ \COMMENT{$\u_t$ sampled from uniform dist. over the unit sphere} 
    \STATE Play $\z_t $ and observe $f_{t}(\z_t), g_t(\z_t)$ 
    \STATE  $\nabla_{t,\x} \gets \frac{n}{\mu}\left({f_t(\z_t) + \lambda_m{}g_t^+(\z_t)}\right)\u_t$ and $\nabla_{t,\lambda} \gets g_t^+(\z_t) -\delta \eta$ 
\ENDFOR
    \STATE Denote $\mathcal{T}_m = \{ (m-1)K +1 , \dots, mK \}$
    \STATE  $\nabla_{m,\x} \gets  \sum_{t\in \mathcal{T}_m} \nabla_{t,\x}$ and $\nabla_{m,\lambda} \gets \sum_{t\in \mathcal{T}_m} \nabla_{t,\lambda}$
    \STATE  $\y_{m+1} \gets \Pi_{R\ball}\brac{\tilde{\y}_m - \eta \nabla_{m,\x}}$
    \STATE $\brac{\x_{m+1},\tilde{\y}_{m+1}}\gets \mathcal{O}_{AFP}\brac{\y_{m+1},\x_m,\epsilon,\mK_{\mu/r}}$    \COMMENT{Apply Algorithm  \ref{alg:CIP-FW} w.r.t. the set $\mK_{\mu/r}$}   
     \STATE  $\lambda_{m+1} \gets \Pi_{\reals_{+}} \brac{\lambda_m  + \eta \nabla_{m,\lambda}}$
\ENDFOR
\end{algorithmic}
\caption{ LOO-based Bandit Primal-Dual Method } \label{alg:LFbandit}
\end{algorithm}

\begin{theorem}\label{thm:bandit}
Suppose  $\widehat{\mK} = \{ \x \in \mK :  g_t(\x) \leq 0~ \forall t\in [T] \}\neq \emptyset$. For any $T$ sufficiently large, setting $\delta =  24\sqrt{183}\max\{G_f,G_g,M_f,M_g\}^2\max\{\frac{n}{r}, \frac{n}{\sqrt{r}}, \sqrt{n}, \frac{n}{\sqrt{r}R}\}T^{\frac{1}{2}}\sqrt{\log{\brac{T}}}$, $\eta = \frac{R}{\max\{G_f,G_g,M_f,M_g\}}T^{-\frac{3}{4}}$, $\epsilon = 61 R^2 T^{-\frac{1}{2}} \log{ T }$ , $K = T^\frac{1}{2}$, $\mu=\sqrt{nr}T^{-1/4}$ in Algorithm \ref{alg:LFbandit}, guarantees that the expected regret is upper bounded by
\begin{align*}
    &\E\left[{\sum_{t=1}^{T} f_t(\z_{t})}\right] - \min_{\x \in \widehat{\mK}} \sum_{t=1}^{T} f_t(\x) = \\
    &O\left({\sqrt{n}\max\{G_f,G_g,M_f,M_g\}T^{3/4}\left({\max\{R/r,R,1/\sqrt{r}\} + R(1/\sqrt{n}+1/\sqrt{r})\sqrt{\log{T}}}\right)}\right),
\end{align*}
 the expected constraints violation is upper bounded by
\begin{align*}
 \E\left[{\sum_{t=1}^{T}  g_t^+(\z_t)}\right] = O\left({\sqrt{RG_f\max\{G_f,G_g,M_f,M_g\}\max\{\frac{R}{r}, \frac{R}{\sqrt{r}},\frac{R}{\sqrt{n}}, \frac{1}{\sqrt{r}}, \frac{1}{R}\}}n^{1/4}T^{7/8}{\log(T)^{1/4}}}\right),
 \end{align*}
and the overall number of calls to the LOO is upper bounded by $T$.
\end{theorem}

Before we prove the theorem we shall require some preliminary results.

Throughout this section we  use the following notation:
\begin{align*}
\mathcal{L}_t(\x,\lambda) &:= f_t(\x) + \lambda g_t^+(\x) - \frac{\delta \eta}{2} \lambda^2 \quad \forall t\in[T], \\
\widehat{\mathcal{L}}_{t}(\x,\lambda) &:= \E_{\u\sim{}U(\mS^n)}\left[{f_t(\x+\mu\u) + \lambda g_t^+(\x+\mu\u)}\right] - \frac{\delta \eta}{2} \lambda^2\quad   \forall t\in[T],\\
\mathcal{L}_{m}(\x,\lambda) &:= \sum_{t\in\mathcal{T}_m}\mathcal{L}_{t}(\x,\lambda)\quad \textrm{and} \quad \widehat{\mathcal{L}}_{m}(\x,\lambda) := \sum_{t\in\mathcal{T}_m}\widehat{\mathcal{L}}_{t}(\x,\lambda) \quad \forall m\in[T/K],
\end{align*}
where we recall that $\mathcal{T}_m = \{(m-1)K+1,\dots,mK\}$.

In the following Lemma \ref{lem:banditFunc} and Lemma \ref{lem:banditMoment}, the notation $\nabla_{m,\x}$ is as defined in Algorithm \ref{alg:LFbandit}.

\begin{lemma}\label{lem:banditFunc}
For every $m\in[T/K]$ the following holds:
\begin{align*}
\forall \x\in\mK_{\mu/r}, ~\lambda\in\reals_+: \quad 0 \leq  \widehat{\mathcal{L}}_{m}(\x,\lambda) - \mathcal{L}_{m}(\x,\lambda)  \leq \mu{}K(G_f+\lambda{}G_g), 
\end{align*}
\begin{align*}
\nabla_{\x}\widehat{\mathcal{L}}_{m}(\x_m,\lambda_m) = \E_{\{\u_t\}_{t\in\mathcal{T}_m}}\left[{\nabla_{m,\x}}\right] \quad \textrm{and} \quad  \nabla_{\lambda}\widehat{\mathcal{L}}_{m}(\x_m,\lambda_m) = \sum_{t\in\mathcal{T}_m}\E_{\u\sim{}U(\mS^n)}\left[{g_t^+(\x+\mu\u)}\right] - \delta \eta K \lambda.
\end{align*}
\end{lemma}
\begin{proof}
The second part of the lemma is a well known result, see for instance \cite{flaxman2005online} and \cite{HazanBook} (Lemma 6.7). 
As for the first part, it is also well known that 
\begin{align*}
\forall \x\in\mK_{\mu/r}, ~\lambda\in\reals_+: \quad \vert{\widehat{\mathcal{L}}_{m}(\x,\lambda) - \mathcal{L}_{m}(\x,\lambda)}\vert  \leq \mu{}K(G_f+\lambda{}G_g),
\end{align*}
see \cite{flaxman2005online} and \cite{HazanBook} (Lemma 2.8). One side of this bound could be improved, as suggested by the lemma, by noticing that for any convex function $h:\mK\rightarrow\reals$ and any $\x\in\mK_{\mu/r}$ we have that,
\begin{align*}
h(\x) - \E_{\u\sim{}U(\mS^n)}[h(\x+\mu\u)] = \E_{\u\sim{}U(\mS^n)}[h(\x) - h(\x+\mu\u)]  \leq \E_{\u\sim{}U(\mS^n)}[-\mu\u^{\top}\g_{\x}] =
-\mu\E_{\u\sim{}U(\mS^n)}[\u]^{\top}\g_{\x} = 0,
\end{align*}
where $\g_{\x}$ is some subgradient in $\partial{}h(\x)$.
\end{proof}
\begin{lemma}[Lemma 5 in \cite{garber2020improved}]\label{lem:banditMoment}
For every $m\in[T/K]$, 
\begin{align*}
\E[\Vert{\nabla_{m,\x}}\Vert]^2 &\leq \E[\Vert{\nabla_{m,\x}}\Vert^2] \leq K\left({\frac{n(M_f+\lambda_m{}M_g)}{\mu}}\right)^2 + K^2(G_f+\lambda_m{}G_g)^2.
\end{align*}
\end{lemma}

The following lemma is analogous to Lemma \ref{lemma:LF} in the full-information setting.
\begin{lemma}\label{lemma:LFbandit}
Algorithm \ref{alg:LFbandit} guarantees that for every $\x \in \mK_{\mu/r}$ and $\lambda \in \reals_+$, 
\begin{align*}
\E\left[\sum_{t=1}^T{\mathcal{L}_t(\x_{m(t)},\lambda) - \mathcal{L}_t(\x,\lambda_{m(t)})}\right] &\leq
 \left({\frac{\sqrt{3\epsilon}\sqrt{K}nM_g}{\mu}+KG_g(\sqrt{3\epsilon}+\mu)}\right)\sum_{m=1}^{T/K}\E[\lambda_m] \\
 &~~+ \left({\frac{\eta{}Kn^2M_g^2}{\mu^2}+\eta{}K^2G_g^2 + K^2\delta^2\eta^3}\right)\sum_{m=1}^{T/K}\E[\lambda_m^2] \\
 &~~+ \frac{T}{K}\left({\frac{\sqrt{3\epsilon}\sqrt{K}nM_f}{\mu}+KG_f(\sqrt{3\epsilon}+\mu) }\right. \\
 &~~+ \left.{\eta\frac{Kn^2M_f^2}{\mu^2}+\eta{}K^2(G_f^2+M_g^2)}\right) + \frac{4R^2 + \lambda^2}{2\eta}.
\end{align*}
\end{lemma}
\begin{proof}
Fix some block index $m$ in Algorithm \ref{alg:LFbandit}. 
Using Lemma \ref{lemma:CIP-FW} and the facts that $\y_{m+1} = \Pi_{R\ball}\brac{\tilde{\y}_m - \eta \nabla_{m,\x}}$ and $\mK_{\mu/r}\subseteq\mK \subseteq R\ball$, we have that for every $\x \in \mK_{\mu/r}$ it holds that
\begin{align*}
    \enorm{ \tilde{\y}_{m+1} - \x}^2 \leq  \enorm{ \y_{m+1}- \x}^2 \leq \enorm{\tilde{\y}_m -\eta \nabla_{m,\x}-\x}^2  = \enorm{\tilde{\y}_m-\x}^2 -2\eta \nabla_{m,\x}^\top \brac{\tilde{\y}_m - \x} + \eta^2 \enorm{\nabla_{m,\x}}^2,
\end{align*}
and rearranging, we have that
\begin{align}
    \nabla_{m, \x}^\top \brac{ \tilde{\y}_m - \x}  & \leq \frac{1}{2\eta} \brac{ \enorm{ \tilde{\y}_m - \x}^2 - \enorm{\tilde{\y}_{m+1} -\x}^2} + \frac{\eta}{2} \enorm{\nabla_{m,\x}}^2. \label{eq:bound_grad_x:bandit}
\end{align}
Since $\lambda_{m+1} = \Pi_{\reals_{+}} \brac{\lambda_m  + \eta \nabla_{m,\lambda}}$, for every $\lambda \in \reals_+$ it holds that
\begin{align*}
    \brac{\lambda_{m+1}-\lambda}^2 \leq \brac{ \lambda_m + \eta \nabla_{m,\lambda} - \lambda}^2 = \brac{\lambda_{m}- \lambda}^2 + 2 \eta \nabla_{m ,\lambda} \brac{ \lambda_m- \lambda } +\eta^2 \nabla_{m,\lambda}^2,
\end{align*}
and rearranging, we have that  
\begin{align}
     \nabla_{m, \lambda} \brac{ \lambda - \lambda_m}\leq \frac{1}{2\eta} \brac{\brac{ \lambda_{m} - \lambda}^2 -  \brac{ \lambda_{m+1} - \lambda }^2} +\frac{\eta}{2} \nabla_{m,\lambda}^2. \label{eq:bound_grad_lambda:bandit}
\end{align}
Let us now fix some $t\in[T]$.
Since $\widehat{\mathcal{L}}_t\brac{\cdot,\lambda}$ is convex, for all $\x\in\mK_{\mu/r}$ it holds that,
\begin{align*}
    \widehat{\mathcal{L}}_t \brac{\x_{m(t)},\lambda_{m(t)}} - \widehat{\mathcal{L}}_t \brac{\x,\lambda_{m(t)}} \leq \nabla_{\x}\widehat{\mathcal{L}}_t\brac{\x_{m(t)},\lambda_{m(t)}}^\top \brac{\x_{m(t)} - \x},
\end{align*}
and since $\widehat{\mathcal{L}}_t\brac{\x_{m(t)} , \cdot}$ is concave, for all $\lambda\in\reals_+$ it holds that,
\begin{align*}
    \widehat{\mathcal{L}}_t \brac{\x_{m(t)},\lambda} - \widehat{\mathcal{L}}_t \brac{\x_{m(t)},\lambda_{m(t)}} \leq \nabla_{\lambda}\widehat{\mathcal{L}}_t\brac{\x_{m(t)},\lambda_{m(t)}} \brac{\lambda-\lambda_{m(t)} }.
\end{align*}
Combining the last two inequalities we have that for any $\x\in\mK_{\mu/r}$ and $\lambda\in\reals^+$ it holds that,
\begin{align*}
    \widehat{\mathcal{L}}_t \brac{\x_{m(t)},\lambda}  - \widehat{\mathcal{L}}_t \brac{\x,\lambda_{m(t)}} \leq \nabla_{\x}\widehat{\mathcal{L}}_t\brac{\x_{m(t)},\lambda_{m(t)}}^\top \brac{\x_{m(t)} - \x} + \nabla_{\lambda}\widehat{\mathcal{L}}_t\brac{\x_{m(t)},\lambda_{m(t)}} \brac{\lambda-\lambda_m }.
\end{align*}

Let us introduce the notation $\E_m[\cdot] = \E_{\{\u_t\}_{t\in\mathcal{T}_m}}[\cdot]$, where we recall that $\mathcal{T}_m = \{(m-1)K+1,\dots,mK\}$. That is, $\E_m[\cdot]$ denotes the expectation w.r.t. all randomness introduces by the random vectors $\{\u_t\}_{t\in\mathcal{T}_m}$ on some block $m$ during the run of the algorithm. Let us also recall that according to Lemma \ref{lem:banditFunc}, $\E_m[\nabla_{m,\x}] = \nabla_{\x}\widehat{\mathcal{L}}_m(\x_m,\lambda_m) = \sum_{t\in\mathcal{T}_m}\nabla_{\x}\widehat{\mathcal{L}}_t(\x_m,\lambda_m)$ and $\E_m[\nabla_{m,\lambda}] = \nabla_{\lambda}\widehat{\mathcal{L}}_m(\x_m,\lambda_m) = \sum_{t\in\mathcal{T}_m}\nabla_{\lambda}\widehat{\mathcal{L}}_t(\x_m,\lambda_m)$. 

Summing the above inequality over $t\in\mathcal{T}_m$, we have that for any $\x\in\mK_{\mu/r}$ and $\lambda\in\reals^+$ it holds that,
\begin{align*}
    \sum_{t\in\mathcal{T}_m} \widehat{\mathcal{L}}_t \brac{\x_m ,\lambda}  - \widehat{\mathcal{L}}_t \brac{\x,\lambda_m} & \leq \E_m[\nabla_{m ,\x}]^\top \brac{\x_m - \x} + \E_m[\nabla_{m ,\lambda}] \brac{\lambda-\lambda_m }\\
    & = \E_m[\nabla_{m ,\x}]^\top \brac{\x_m - \tilde{\y}_m} + \E_m[\nabla_{m ,\x}]^\top \brac{\tilde{\y}_m - \x} + \E_m[\nabla_{m ,\lambda}] \brac{\lambda-\lambda_m }.
\end{align*}
Combining the last inequality with Eq. \eqref{eq:bound_grad_x:bandit} and \eqref{eq:bound_grad_lambda:bandit}, and noting that $\x_m,\y_m,\lambda_m$ are independent of the random vectors $\{\u_t\}_{t\in\mathcal{T}_m}$, we have that for every $\x \in \mK_{\mu/r}$ and $\lambda \in \reals_+$ it holds that,
\begin{align*}
    \sum_{t\in\mathcal{T}_m} \widehat{\mathcal{L}}_t \brac{\x_m ,\lambda}  - \widehat{\mathcal{L}}_t \brac{\x,\lambda_m} & \leq 
    \E_m[\nabla_{m ,\x}^\top \brac{\x_m - \tilde{\y}_m}]\\
    &~~~ + \frac{1}{2\eta} \E_m\left[{ \enorm{\tilde{\y}_m - \x}^2 - \enorm{\tilde{\y}_{m+1}- \x}^2  }\right] + \frac{\eta}{2}\E_m[ \enorm{\nabla_{m,\x}}^2] \\
    &~~~+ \frac{1}{2\eta}\E_m\left[{\brac{\lambda_{m} - \lambda}^2 -  \brac{ \lambda_{m+1} - \lambda}^2 }\right] + \frac{\eta}{2}\E_m[\nabla_{m,\lambda}^2].
\end{align*}
Summing the above over $m\in\{1,\dots,T/K\}$ and taking expectation w.r.t. all random variables $\u_1,\dots,\u_T$, we have for every $\x \in \mK_{\mu/r}$ and $\lambda \in \reals_+$ that,
\begin{align*}
    \E\left[{\sum_{m=1}^{T/K} \sum_{t\in\mathcal{T}_m} \widehat{\mathcal{L}}_t \brac{\x_m,\lambda}  - \widehat{\mathcal{L}}_t \brac{\x,\lambda_m}}\right] & \leq  \sum_{m=1}^{T/K}\E\left[{ \enorm{\nabla_{m ,\x}} \enorm{\x_m - \tilde{\y}_m}}\right]  \\
    &~+  \frac{\eta}{2} \sum_{m=1}^{T/K}\brac{\E[\enorm{\nabla_{m,\x}}^2]+\E[\nabla_{m,\lambda}^2] }  
     + \frac{\enorm{\tilde{\y}_1 - \x}^2 + \brac{ \lambda_{1} - \lambda}^2}{2\eta} .
\end{align*}
Using Lemma \ref{lemma:CIP-FW} we have that $\enorm{\x_m -\tilde{\y}_m}^2\leq 3\epsilon$ for every $m$. Also, since $\tilde{\y}_1 \in \mK_{\mu/r}\subseteq{}R\ball$ and $\lambda_1 = 0$, for every $\x \in \mK_{\mu/r}$ and  $\lambda \in \reals_+$ we have  that, 
\begin{align}
   \E\left[{ \sum_{t=1}^T\widehat{\mathcal{L}}_t \brac{\x_{m(t)},\lambda}  - \widehat{\mathcal{L}}_t \brac{\x,\lambda_{m(t)}}}\right] & \leq \sqrt{3\epsilon} \sum_{m=1}^{T/K} \E[\enorm{\nabla_{m ,\x}}]  + \frac{4R^2 + \lambda^2}{2\eta} \nonumber\\
   &~ +  \frac{\eta}{2} \sum_{m=1}^{T/K}\brac{\E[\enorm{\nabla_{m,\x}}^2]+\E[\nabla_{m,\lambda}^2] } .\label{eq:lagrangian_regret_1:bandit}
\end{align}
Using Lemma \ref{lem:banditMoment} we have that for every $m\in[T/K]$,
\begin{align*}
\E_m[\Vert{\nabla_{m,\x}}\Vert]^2 &\leq \E_m[\Vert{\nabla_{m,\x}}\Vert^2] \leq K\left({\frac{n(M_f+\lambda_m{}M_g)}{\mu}}\right)^2 + K^2(G_f+\lambda_m{}G_g)^2.
\end{align*}
Recalling also that $\nabla_{m,\lambda} = \sum_{t\in\mathcal{T}_m}\left({ g_t^+(\z_t) -\delta \eta \lambda_m}\right)$, we obtain
\begin{align}\label{eq:lagrangeBandit:lem:1}
    \E\left[{\sum_{t=1}^{T} \widehat{\mathcal{L}}_t \brac{\x_{m(t)},\lambda}  - \widehat{\mathcal{L}}_t \brac{\x,\lambda_{m(t)}}}\right] & \leq  
\sqrt{3\epsilon}\sum_{m=1}^{T/K}\E\left[{\sqrt{K\left({\frac{n(M_f+\lambda_m{}M_g)}{\mu}}\right)^2 + K^2(G_f+\lambda_m{}G_g)^2}}\right] \nonumber \\
&~~+   \frac{\eta}{2}\sum_{m=1}^{T/K}\E\left[{K\left({\frac{n(M_f+\lambda_m{}M_g)}{\mu}}\right)^2 + K^2(G_f+\lambda_m{}G_g)^2}\right] \nonumber\\
&~~+\frac{\eta}{2}\sum_{m=1}^{T/K}\E\left[{\left({\sum_{t\in\mathcal{T}_m}(g_t^+(\z_t) - \delta\eta\lambda_m)}\right)^2}\right]
+ \frac{4R^2 + \lambda^2}{2\eta} \nonumber\\
&\leq \sqrt{3\epsilon}\sum_{m=1}^{T/K}\E\left[{\sqrt{K}\frac{n(M_f + \lambda_mM_g)}{\mu} + K(G_f+\lambda_mG_g)}\right] \nonumber \\
&~~+\frac{\eta}{2}\sum_{m=1}^{T/K}\E\left[{K\left({\frac{n(M_f+\lambda_m{}M_g)}{\mu}}\right)^2 + K^2(G_f+\lambda_m{}G_g)^2}\right] \nonumber\\
&~~+\eta\sum_{m=1}^{T/K}\E\left[{\left({\sum_{t\in\mathcal{T}_m}g_t^+(\z_t)}\right)^2 + K^2\delta^2\eta^2\lambda_m^2}\right]
+ \frac{4R^2 + \lambda^2}{2\eta}, 
\end{align}
where in the last inequality we have used the facts that $(a+b)^2 \leq 2a^2+2b^2$ for any $a,b\in\reals$, and $\sqrt{a^2+b^2} \leq a+b$ for any $a,b\in\reals_+$.

Rearranging the RHS in the above using the fact that  $g_t^+(\z_t) \leq M_g$ for all $t\in[T]$, and also using again the inequaity $(a+b)^2 \leq 2a^2 + 2b^2$ gives,
\begin{align*}
\textrm{RHS of } \eqref{eq:lagrangeBandit:lem:1} &\leq
 \sqrt{3\epsilon}\left({\frac{\sqrt{K}nM_g}{\mu}+KG_g}\right)\sum_{m=1}^{T/K}\E[\lambda_m]
 + \left({\frac{\eta{}Kn^2M_g^2}{\mu^2}+\eta{}K^2G_g^2+  K^2\delta^2\eta^3}\right)\sum_{m=1}^{T/K}\E[\lambda_m^2] \\
 &~~+ \frac{T}{K}\left({\frac{\sqrt{3\epsilon}\sqrt{K}nM_f}{\mu}+\sqrt{3\epsilon}KG_f + \eta\frac{Kn^2M_f^2}{\mu^2}+\eta{}K^2(G_f^2+M_g^2)}\right) + \frac{4R^2 + \lambda^2}{2\eta}.
\end{align*}

Finally, using Lemma \ref{lem:banditFunc}, for any $\x\in\mK_{\mu/r}$ and $\lambda\in\reals_+$ we have that,
\begin{align*}
\E\left[\sum_{t=1}^T{\mathcal{L}_t(\x_{m(t)},\lambda) - \mathcal{L}_t(\x,\lambda_{m(t)})}\right] &\leq
 \left({\frac{\sqrt{3\epsilon}\sqrt{K}nM_g}{\mu}+KG_g(\sqrt{3\epsilon}+\mu)}\right)\sum_{m=1}^{T/K}\E[\lambda_m] \\
 &~~+ \left({\frac{\eta{}Kn^2M_g^2}{\mu^2}+\eta{}K^2G_g^2 + K^2\delta^2\eta^3}\right)\sum_{m=1}^{T/K}\E[\lambda_m^2] \\
 &~~+ \frac{T}{K}\left({\frac{\sqrt{3\epsilon}\sqrt{K}nM_f}{\mu}+KG_f(\sqrt{3\epsilon}+\mu) }\right. \\
 &~~+ \left.{\eta\frac{Kn^2M_f^2}{\mu^2}+\eta{}K^2(G_f^2+M_g^2)}\right) + \frac{4R^2 + \lambda^2}{2\eta}.
\end{align*}


%

\end{proof}

\begin{proof}[Proof of Theorem \ref{thm:bandit}]
From Lemma \ref{lemma:LF} we have  for every $\x \in \mK_{\mu/r}$ and $\lambda \in \reals_+$ that,
\begin{align*}
   & \E\left[{\sum_{t=1}^{T}  \brac{f_t(\x_{m(t)}) -f_t(\x)} + \sum_{t=1}^{T} \brac{\lambda g_t^+(\x_{m(t)}) - \lambda_{m(t)} g_t^+(\x) } + \frac{\delta \eta K }{2} \sum_{m=1}^{T/K} \lambda_m^2 - \frac{\delta \eta }{2} T \lambda^2}\right]\\
   &~~\leq \left({\frac{\sqrt{3\epsilon}\sqrt{K}nM_g}{\mu}+KG_g(\sqrt{3\epsilon}+\mu)}\right)\sum_{m=1}^{T/K}\E[\lambda_m] + \left({\frac{\eta{}Kn^2M_g^2}{\mu^2}+\eta{}K^2G_g^2 + K^2\delta^2\eta^3}\right)\sum_{m=1}^{T/K}\E[\lambda_m^2] \\
 &~~+ \frac{T}{K}\left({\frac{\sqrt{3\epsilon}\sqrt{K}nM_f}{\mu}+\sqrt{3\epsilon}KG_f + \eta\frac{Kn^2M_f^2}{\mu^2}+\eta{}K^2(G_f^2+M_g^2) + \mu{}K{}G_f}\right)+ \frac{4R^2 + \lambda^2}{2\eta}.
 \end{align*}

Rearranging, we have that for every $\x \in \mK_{\mu/r}$ and $\lambda \in \reals_+$ it holds that,
\begin{align*}
    &\E\left[{ \sum_{t=1}^{T} \brac{f_t(\x_{m(t)}) -f_t(\x)} + \sum_{t=1}^{T}\lambda g_t^+(\x_{m(t)}) }\right] - \brac{\frac{\delta \eta }{2} T + \frac{1}{2\eta}} \lambda^2   \\ 
    & ~~~ \leq  \frac{2R^2}{\eta}   + \left({\frac{\sqrt{3\epsilon}\sqrt{K}nM_g}{\mu}+KG_g(\sqrt{3\epsilon}+\mu)}\right)\sum_{m=1}^{T/K}\E[\lambda_m] \\
    &~~~+  \left({\frac{\eta{}Kn^2M_g^2}{\mu^2}+\eta{}K^2G_g^2 + K^2\delta^2\eta^3 - \frac{\delta\eta{}K}{2}}\right)\sum_{m=1}^{T/K}\E[\lambda_m^2] \\
    &~~~+T\left({\frac{\sqrt{3\epsilon}nM_f}{\mu\sqrt{K}}+\sqrt{3\epsilon}G_f + \eta\frac{n^2M_f^2}{\mu^2}+\eta{}K(G_f^2+M_g^2) + \mu{}G_f}\right) + \sum_{t=1}^Tg_t^+(\x)\E[\lambda_{m(t)}]. 
\end{align*}

Fix some $\x^*\in\widehat{\mK}$ such that $\x^*\in\arg\min_{\w\in\widehat{\mK}}\sum_{t=1}^Tf_t(\w)$ and from now on fix $\x=(1-\mu/r)\x^*\in\mK_{\mu/r}$. Note that for all $t\in[T]$, using the convexity of $g_t^+$  and the fact that $g_t^+(\x^*) = 0$ we have that,
\begin{align*}
g_t^+(\x) = g_t^+((1-\mu/r)\x^* + (\mu/r)\mathbf{0}) \leq (1-\mu/r)g_t^+(\x^*) + (\mu/r)g_t^+(\mathbf{0}) \leq \frac{\mu}{r}M_g. 
\end{align*}
This gives,
\begin{align*}
    &\E\left[{ \sum_{t=1}^{T} \brac{f_t(\x_{m(t)}) -f_t(\x)} + \sum_{t=1}^{T}\lambda g_t^+(\x_{m(t)}) }\right] - \brac{\frac{\delta \eta }{2} T + \frac{1}{2\eta}} \lambda^2   \\ 
    & ~~~ \leq  \frac{2R^2}{\eta}   + \left({\frac{\sqrt{3\epsilon}\sqrt{K}nM_g}{\mu}+KG_g(\sqrt{3\epsilon}+\mu) + \frac{\mu{}KM_g}{r}}\right)\sum_{m=1}^{T/K}\E[\lambda_m] \\
    &~~~+  \left({\frac{\eta{}Kn^2M_g^2}{\mu^2}+\eta{}K^2G_g^2 + K^2\delta^2\eta^3 - \frac{\delta\eta{}K}{2}}\right)\sum_{m=1}^{T/K}\E[\lambda_m^2] \\
    &~~~+T\left({\frac{\sqrt{3\epsilon}nM_f}{\mu\sqrt{K}}+\sqrt{3\epsilon}G_f + \eta\frac{n^2M_f^2}{\mu^2}+\eta{}K(G_f^2+M_g^2) + \mu{}G_f}\right).
\end{align*}

Suppose throughout the rest of the proof that the following condition holds:
\begin{align*}
\frac{\delta\eta{}K}{2} &\geq \frac{\eta{}Kn^2M_g^2}{\mu^2}+\eta{}K^2G_g^2 + K^2\delta^2\eta^3 + \frac{\sqrt{3\epsilon}\sqrt{K}nM_g}{\mu}+KG_g(\sqrt{3\epsilon}+\mu) + \frac{\mu{}KM_g}{r}\\
&=\eta{}K\left({\frac{n^2M_g^2}{\mu^2} + KG_g^2 +\eta^2K\delta^2}\right) + \sqrt{3\epsilon{}K}\left({\frac{nM_g}{\mu} + \sqrt{K}G_g}\right) + 
K\mu\left({G_g + \frac{M_g}{r}}\right).
\end{align*}
In particular, for the above to hold it suffices that the following holds:
\begin{align}\label{eq:thm:bandit:1}
\delta \geq  12\max\left\{\frac{n^2M_g^2}{\mu^2}, KG_g^2, K\delta^2\eta^2, \frac{\sqrt{3\epsilon}nM_g}{\eta\mu\sqrt{K}}, \frac{G_g(\sqrt{3\epsilon}+\mu)}{\eta}, \frac{\mu{}M_g}{r\eta}\right\}.
\end{align}

Then we can write,
\begin{align*}
    &\E\left[{ \sum_{t=1}^{T} \brac{f_t(\x_{m(t)}) -f_t(\x)} + \sum_{t=1}^{T}\lambda g_t^+(\x_{m(t)}) }\right] - \brac{\frac{\delta \eta }{2} T + \frac{1}{2\eta}} \lambda^2   \\ 
    & ~~~ \leq  \frac{2R^2}{\eta}   + \left({\frac{\sqrt{3\epsilon}\sqrt{K}nM_g}{\mu}+KG_g(\sqrt{3\epsilon}+\mu) + \frac{\mu{}KM_g}{r}}\right)\sum_{m=1}^{T/K}\E[\lambda_m - \lambda_m^2] \\
    &~~~+T\left({\frac{\sqrt{3\epsilon}nM_f}{\mu\sqrt{K}}+\sqrt{3\epsilon}G_f + \eta\frac{n^2M_f^2}{\mu^2}+\eta{}K(G_f^2+M_g^2) + \mu{}G_f}\right).
\end{align*}

Since $z-z^2 \leq \frac{1}{4}$ for every $z\in\reals$, it holds that
\begin{align}\label{eq:bandit:thm:1}
    &\E\left[{ \sum_{t=1}^{T} \brac{f_t(\x_{m(t)}) -f_t(\x)} + \sum_{t=1}^{T}\lambda g_t^+(\x_{m(t)}) }\right] - \brac{\frac{\delta \eta }{2} T + \frac{1}{2\eta}} \lambda^2 \leq  \frac{2R^2}{\eta} \nonumber  \\ 
    & ~~~    + \frac{T}{4}\left({\frac{\sqrt{3\epsilon}n(M_g+4M_f)}{\mu\sqrt{K}}+G_g(\sqrt{3\epsilon}+\mu) + \frac{\mu{}M_g}{r} + 4\sqrt{3\epsilon}G_f + 4\eta\frac{n^2M_f^2}{\mu^2}+4\eta{}K(G_f^2+M_g^2) + 4\mu{}G_f}\right) .
\end{align}
Note that $\lambda^* = \frac{ \sum_{t=1}^{T}\E[  g_t^+(\x_{m(t)})]}{\delta \eta  T  + \eta^{-1}}$ maximizes the term $\lambda\sum_{t=1}^{T}\E[  g_t^+(\x_{m(t)})] - \brac{\frac{\delta \eta }{2} T  + \frac{1}{2\eta}} \lambda^2$ in the above inequality as a function of $\lambda$. Plugging-in $\x=(1-\mu/r)\x^*$ and $\lambda^*$ into the last inequality, we have that
\begin{align}\label{eq:bandit:thm:2}
    \E\left[{\sum_{t=1}^{T} \brac{f_t(\x_{m(t)}) -f_t((1-\mu/r)\x^*)}}\right] + \frac{1}{2} \frac{ \brac{\sum_{t=1}^{T}  \E[g_t^+(\x_{m(t)})]}^2}{\delta \eta  T  + \eta^{-1}} \leq \textrm{RHS of } \eqref{eq:bandit:thm:1}.
\end{align}

Suppose $\eta = c_1T^{-3/4}, K=T^{1/2}, \epsilon=61R^2T^{-1/2}\log{T}, \delta = c_2T^{1/2}\sqrt{\log{T}}$ and $\mu=c_3T^{-1/4}$ for some constants $c_1,c_2,c_3$. This gives
\begin{align*}
\textrm{RHS of } \eqref{eq:bandit:thm:1} &= 
O\left({\frac{R^2T^{3/4}}{c_1} + T\left({\frac{RT^{-1/4}\sqrt{\log{T}}n(M_g+M_f)}{c_3} + G_g\left({RT^{-1/4}\sqrt{\log{T}} + c_3T^{-1/4}}\right)  }\right.}\right. \\
&\left.{\left.{  +\frac{c_3T^{-1/4}M_g}{r} + G_fRT^{-1/4}\sqrt{\log{T}} + \frac{c_1T^{-1/4}n^2M_f^2}{c_3^2} +c_1T^{-1/4}(G_f^2+M_g^2) + c_3G_fT^{-1/4} }\right)}\right) \\
&=O\left({T^{3/4}\sqrt{\log{T}}\left({\frac{nR(M_g+M_f)}{c_3} + R(G_g+G_f)}\right) }\right.\\
&~~~\left.{+T^{3/4}\left({\frac{R^2}{c_1} + c_3\left({G_f+G_g+\frac{M_g}{r}}\right)+\frac{c_1n^2M_f^2}{c_3^2} + c_1(G_f^2+M_g^2)}\right)}\right).
\end{align*}

Setting $c_1 = \frac{R}{\sqrt{n}\max\{G_f,G_g,M_f,M_g\}}$, $c_3 = \sqrt{nr}$, we obtain
\begin{align*}
\textrm{RHS of } \eqref{eq:bandit:thm:1} &= O\Bigg({T^{3/4}\sqrt{\log{T}}\left({\sqrt{\frac{n}{r}}R(M_g+M_f) + R(G_g+G_f)}\right) }\Bigg.\\
&~~~ \Bigg.{+ T^{3/4}\left({\sqrt{n}\max\{R/r, R,1/\sqrt{r}\}\max\{G_f,G_g,M_f,M_g\}}\right)}\Bigg) \\
&= O\left({\sqrt{n}\max\{G_f,G_g,M_f,M_g\}T^{3/4}\left({\max\{R/r,R,1/\sqrt{r}\} + R(1/\sqrt{n}+1/\sqrt{r})\sqrt{\log{T}}}\right)}\right).
\end{align*}

Note that for all $t\in[T]$, using the convexity of $f_t$, we have that
\begin{align*}
&f_t(\z_t) = f_t(\x_{m(t)}+\mu\u_t) \leq f_t(\x_{m(t)}) + \mu{}G_f, \quad f_t((1-\mu/r)\x^*)  \leq f_t(\x^*)+  \frac{\mu{}G_f}{r}.
\end{align*}
Plugging these into \eqref{eq:bandit:thm:2} yields the regret bound:
\begin{align*}
    &\E\left[{\sum_{t=1}^{T} f_t(\z_t) -f_t(\x^*)}\right] \leq \textrm{RHS of } \eqref{eq:bandit:thm:1} + T\mu{}G_f\left({1 + \frac{1}{r}}\right)\\
    &=  O\left({\sqrt{n}\max\{G_f,G_g,M_f,M_g\}T^{3/4}\left({\max\{R/r,R,1/\sqrt{r}\} + R(1/\sqrt{n}+1/\sqrt{r})\sqrt{\log{T}}}\right)}\right).
\end{align*}

Using the fact that for all $t\in[T]$, $f_t((1-\mu/r)\x^*) - f_t(\x_{m(t)}) \leq 2RG_f$ and  $g_t^+(\z_t) \leq g_t^+(\x_{m(t)}) + \mu{}G_g$, we obtain from \eqref{eq:bandit:thm:2} the constraints violation bound:
\begin{align}\label{eq:thm:bandit:2}
    \E\left[{\sum_{t=1}^{T}  g_t^+(\z_t)}\right] \leq \sqrt{2\brac{\delta \eta  T  + \eta^{-1}} \times \left({\textrm{RHS of }  \eqref{eq:bandit:thm:1} + 2RG_fT}\right)} + \mu{}G_gT.
\end{align}
Let us now get back to the condition in \eqref{eq:thm:bandit:1} which we assumed holds true. Plugging-in our choices of $\eta,\mu,\epsilon,K,\delta$, it suffices that the constant $c_2$ satisfies:
\begin{align*}
c_2 \geq  \frac{12T^{-1/2}}{\sqrt{\log{T}}}\max&\left\{{\frac{nM_g^2T^{1/2}}{r}, G_g^2T^{1/2}, \frac{c_2^2R^2}{n\max\{G_f,G_g,M_f,M_g\}^2}\log{T}, 
}\right.\\
&~~ \left.{\frac{\sqrt{183}n\max\{G_f,G_g,M_f,M_g\}^2T^{1/2}\sqrt{\log{T}}}{\sqrt{r}}, }\right.\\
&~~\left.{ \frac{\sqrt{n}\max\{G_f,G_g,M_f,M_g\}^2T^{1/2}(\sqrt{183}R\sqrt{\log{T}}+\sqrt{nr})}{R}, }\right.\\
&~~\left.{\frac{nT^{1/2}\max\{G_f,G_g,M_f,M_g\}^2}{\sqrt{r}R}}\right\}.
\end{align*}
In particular, for any $T$ large enough (so that the second term inside the max in the RHS is not dominant), it suffices to take 
\begin{align*}
c_2 
&= 24\sqrt{183}\max\{G_f,G_g,M_f,M_g\}^2\max\{\frac{n}{r}, \frac{n}{\sqrt{r}}, \sqrt{n}, \frac{n}{\sqrt{r}R}\}.
\end{align*}
Plugging this choice into \eqref{eq:thm:bandit:2}, and noting that the RHS of   \eqref{eq:bandit:thm:1} scales (as a function of $T$) only as $T^{3/4}\sqrt{\log{T}}$,  we have that for any $T$ large enough,
\begin{align*}
 \E\left[{\sum_{t=1}^{T}  g_t^+(\z_t)}\right] = O\left({\sqrt{RG_f\max\{G_f,G_g,M_f,M_g\}\max\{\frac{R}{r}, \frac{R}{\sqrt{r}},\frac{R}{\sqrt{n}}, \frac{1}{\sqrt{r}}, \frac{1}{R}\}}n^{1/4}T^{7/8}{\log(T)^{1/4}}}\right).
\end{align*}

Finally, since $\epsilon,K$ are the same as in Theorem  \ref{thm:LF}, we have that the same upper-bound on the total number of calls to the linear optimization oracle applies, which concludes the proof.
\end{proof}




\end{document}